\def\year{2019}\relax
\documentclass[letterpaper]{article} 
\usepackage{aaai19}  
\usepackage{times}  
\usepackage{helvet}  
\usepackage{courier}  
\usepackage{url}  
\usepackage{graphicx}  
\usepackage{amsfonts}
\usepackage{amsmath}
\usepackage{makecell,multirow,diagbox} 
\usepackage{multicol} 
\usepackage{arydshln}
\usepackage{booktabs}
\usepackage{algorithm}
\usepackage{algorithmic}
\usepackage{bm}
\usepackage{subeqnarray}
\usepackage{cases}
\newtheorem{theorem}{Theorem}

\newtheorem{proposition}{Proposition}
\newtheorem{proof}{Proof}[section]
\newtheorem{remark}{Remark}
\begin{document}
	%
	\title{A Theoretically Guaranteed Deep Optimization Framework \\for Robust Compressive Sensing MRI}
	\author{Risheng Liu\textsuperscript{1,2,}\thanks{Corresponding Author.}, Yuxi Zhang\textsuperscript{1,2}, Shichao Cheng\textsuperscript{2,3}, Xin Fan\textsuperscript{1,2}, Zhongxuan Luo\textsuperscript{1,2,3,4}\\
		\textsuperscript{1}{DUT-RU International School of Information Science \& Engineering, Dalian University of Technology, Dalian, China}\\
		\textsuperscript{2}{Key Laboratory for Ubiquitous Network and Service Software of Liaoning Province, Dalian, China}\\
		\textsuperscript{3}{School of Mathematical Science, Dalian University of Technology, Dalian, China}\\
		\textsuperscript{4}{Institute of Artificial Intelligence, Guilin University of Electronic Technology, Dalian, China}\\
		\{rsliu, xin.fan, zxluo\}@dlut.edu.cn,
		yuxizhang@mail.dlut.edu.cn,
		shichao.cheng@outlook.com\\
	}

	\maketitle
	\begin{abstract}
		Magnetic Resonance Imaging (MRI) is one of the most dynamic and safe imaging techniques available for clinical applications. However, the rather slow speed of MRI acquisitions limits the patient throughput and potential indi
		cations. Compressive Sensing (CS) has proven to be an efficient technique for accelerating MRI acquisition. The most widely used CS-MRI model, founded on the premise of reconstructing an image from an incompletely filled k-space, leads to an ill-posed inverse problem. In the past years, lots of efforts have been made to efficiently optimize the CS-MRI model. Inspired by deep learning techniques, some preliminary works have tried to incorporate deep architectures into CS-MRI process. Unfortunately, the convergence issues (due to the experience-based networks) and the robustness (i.e., lack real-world noise modeling) of these deeply trained optimization methods are still missing. In this work, we develop a new paradigm to integrate designed numerical solvers and the data-driven architectures for CS-MRI. By introducing an optimal condition checking mechanism, we can successfully prove the convergence of our established deep CS-MRI optimization scheme. Furthermore, we explicitly formulate the Rician noise distributions within our framework and obtain an extended CS-MRI network to handle the real-world nosies in the MRI process. Extensive experimental results verify that the proposed paradigm outperforms the existing state-of-the-art techniques both in reconstruction accuracy and efficiency as well as robustness to noises in real scene. 	 
	\end{abstract}
	
	\section{Introduction}
	Magnetic Resonance Imaging (MRI) is widely utilized in clinical applications because of its none-invasive property and excellent capability in revealing both functional and anatomical information. However, one of the main drawbacks is the inherently slow acquisition speed of MRI in k-space (i.e., Fourier space), due to the limitation of hardwares \cite{lustig2008compressed}. Compressive sensing MRI (CS-MRI) is a commonly used technique allowing fast acquisition at data sampling rate much lower than Nyquist rate without deteriorating the image quality. 
	
	In the process of MR data acquisition, the sparse k-space data $\mathbf{y}$ can be approximated as the following discretized linear system \cite{eksioglu2016decoupled}:
	\begin{equation}
	\mathbf{y}= \mathbf{PFx} + \mathbf{n},
	\label{forwoard}
	\end{equation}
	where $ \mathbf{x} \in\mathbb{R}^N $ is the desired MR image to be reconstructed from observation $ \mathbf{y} \in\mathbb{C}^M (M<N)$, $ \mathbf{F} $ is the Fourier transform, $ \mathbf{P} \in \mathbb{R}^{M\times N}$ denotes the under-sampling operation, and $ \mathbf{n} \in \mathbb{C}^M$ represents the acquisition noises. It is easy to find that estimating $\mathbf{x}$ from Eq.~\eqref{forwoard} is actually ill-posed, due to the singularity of $\mathbf{PF}$. 
	Thus the main challenge for CS-MRI lies in defining proper regularizations and proposing corresponding nonlinear optimization or iterative algorithm for the inverse problem. According to CS theory, the most typical CS-MRI reconstruction techniques attempt to optimize the following nonsmooth regularized model:
	\begin{equation}
	\bm{\alpha}\in
	\arg\min\limits_{\bm{\alpha}}
	\ \frac{1}{2}\|\mathbf{PFA}\bm{\alpha}\!-\!\mathbf{y}\|_2^2\!+
	\!\lambda\|\bm{\alpha}\|_p,
	\label{eq:OriginalModel}
	\end{equation}
	where $\bm{\alpha}$ denotes the sparse code of $\mathbf{x}$, corresponding to the wavelet basis (denoted as $\mathbf{A}$, so we actually have $\mathbf{x}=\mathbf{A}\bm{\alpha}$) of a given inverse wavelet transform (usually implemented by DCT or DWT) and  $\lambda $ indicates the trade-off parameter. In this work, we consider the $\ell_p$ regularization with $p \in(0,1)$, thus actually need to address a challenging nonconvex sparse optimization task.

	\subsection{Related Works}
	In a conventional way, some effects are in exploring the sparse regularization in specific transform domain \cite{qu2012undersampled,lustig2007sparse,gho2010three}. In addition, algorithms based on the nonlocal processing paradigm aim to construct a better sparsity transform, \cite{qu2014magnetic,eksioglu2016decoupled}. Some dictionary learning based methods focus on training a dictionary from reference images in the related subspace \cite{ravishankar2011mr,babacan2011reference,zhan2016fast}. These model-based CS-MRI methods give satisfying performances in details restoration benefiting from the model based data prior. While most of them are based on a sparsity introducing $\ell_1$ norm or $\ell_0$ norm, which have weaker abilities to describe the sparsity in real sense and the insufficient sparse regularization may introduce artifacts. What's more, lacking of data dependent prior, such methods have limitation in handling particular structure of the problem or specific data distribution in real scene.

	Recently, some preliminary studies on deep learning based CS-MRI have attracted great attentions \cite{wang2016accelerating,lee2017deep,schlemper2018deep}. At the expense of sacrificing principled information, such techniques benefit an extremely fast feed-forward process with the aid of deep architecture, but are not as flexible as the model based ones to handle various distributions of data.
	
	To absorb the advantages of principled knowledge and data-driven prior, techniques are proposed integrating the deep architecture into an iterative optimization process \cite{liu2018bridging,liu2018convergence,liu2018enhance,liu2018learning}. By introducing learnable architectures into the Alternating Direction Method of Multipliers (ADMM) iterations, they proposed an ADMM-Net to address the CS-MRI problem. In \cite{diamond2017unrolled} unrolled Optimization with Deep Priors (ODP) is proposed to integrate CNN priors into optimization process. It has been demonstrated that these deep priors can significantly improve the reconstruction accuracy. But unfortunately, due to their naive unrolling schemes (i.e.,  directly replace iterations by architectures), the convergence of these deep optimization based CS-MRI methods cannot be theoretically guaranteed. Even worse, no mechanism is proposed to control the errors generated by nested structures. Thus these methods may completely failed if improper architectures are utilized during iterations.
	
	Except for the aforementioned problems, one concern but yet remain unsolved issue is the noise in MR images under real scene. During the practical acquisition of MR data, both real and imaginary parts of the complex data in k-space may be corrupted with uncorrelated zero-mean Gaussian noise with equal variance \cite{rajan2012adaptive}. Thus the actual acquired magnitude MR image $\mathbf{x}_n$ is given as follows:
	\begin{equation}
	\mathbf{x}_n=\sqrt{(\mathbf{x}_c+\mathbf{n}_1)^2+\mathbf{n}_2^2},
	\label{eq:Rician}
	\end{equation}
	where $\mathbf{x}_c$ denotes the noise-free signal intensity, and $\mathbf{n}_1,\mathbf{n}_2\sim\mathcal{N}(0,\sigma^2)$ represent the uncorrelated real and imaginary components of Gaussian noise, respectively. Notice that the noise in magnitude MR image $\mathbf{x}_n$ no longer follows the Gaussian distribution, but a Rician one.
	
	Respectable effects for removing Rician noise from MR data have been made in several ways \cite{wiest2008rician,manjon2010adaptive,chen2015convex,liu2016variational}, but rather few studies on CS-MRI take into account the presence of actual noise in MR images. Either these CS-MRI methods consider only noiseless MR images or treat the noise as a Gaussian one \cite{sun2016deep,yang2018dagan}. To the best of our knowledge, to date this matter has not been well solved in real sense. 
	
	\subsection{Contributions}
	As discussed above, one of the most important limitations of existing deep optimization based CS-MRI methods (e.g., \cite{sun2016deep} and \cite{diamond2017unrolled}) is the lack of theoretical investigations.
	To address the aforementioned issues, in this paper, we propose a new deep algorithmic framework to optimize the CS-MRI problem in Eq.~\eqref{eq:OriginalModel}. Specifically, by integrating domain knowledge of the task and learning-based architectures, as well as checking the optimal conditions, we can strictly prove the convergence of our generated deep propagations. Moreover, due to the complex noise distributions (i.e., Rician), it is still challenging to directly apply existing CS-MRI methods to tasks in real-world scenarios. Thanks to the flexibility of our paradigm, we further extend the propagation framework to adaptively and iteratively remove Rician noise to guarantee the robustness of the MRI process. Our contributions can be summarized as follows:
	\begin{itemize}
		\item To our best knowledge, this is the first work that could establish  theoretically convergent deep optimization algorithms to efficiently solve the nonconvex and nonsmooth CS-MRI model in Eq.~\eqref{eq:OriginalModel}. Thanks to the automatic checking and feedback mechanism (based on first-order optimality conditions), we can successfully reject improper nested architectures, thus will always propagate toward our desired solutions. 
		\item To address the robustness issues in existing CS-MRI approaches, we also develop an extended deep propagation scheme to simultaneously recover the latent MR images and remove the Rician nosies for real-world MRI process.
		\item Extensive experimental results on real-world benchmarks demonstrate the superiority of the proposed paradigm against state-of-the-art techniques in both reconstruction accuracy and efficiency, as well as the robustness to complex noise pollution.
	\end{itemize}

	\section{Our Paradigm for CS-MRI}
	In this section, we design a theoretically converged deep model to unify various of fundamental factors which affect the performance of CS-MRI. Different from most existing deep approaches, all of the fidelity knowledge, data-driven architecture, manual prior and optimality condition are comprehensively taken into consideration in our deep framework. Meanwhile, the theoretical convergence is guaranteed by our optimal checking mechanism. Inspired by recent studies \cite{liu2018bridging,liu2018convergence,yang2017common}, proximal gradient algorithm is utilized in this work to achieve an efficient optimization.

	\noindent{\textbf{Fidelity Module:}}
	Fidelity plays an important role in revealing the intrinsic genesis of problem, which is commonly adopted in traditional hand-crafted approaches. 
	Actually, by given a reason manner to solve the fidelity term (i.e., $\frac{1}{2}\|\mathbf{PFA}\bm{\alpha}\!-\!\mathbf{y}\|_2^2$ in Eq.~\eqref{eq:OriginalModel}), we can generate a rough reconstruction. 
	Rather than introducing additional complex constrains, here we just integrate the current restoration $\bm{\alpha}^{k}$ with the fidelity and consider the following subproblem:
	\begin{equation}
	\mathbf{u}^{k+1}  =
	\arg\min\limits_{\mathbf{u}}
	\frac{1}{2}\|\mathbf{PFAu}\!-\!\mathbf{y}\|_2^2\!+
	\!\frac{\rho}{2}\|\mathbf{u}\!-\!\bm{\alpha}^k\|_2^2,
	\label{eq:ModelWithProxTerm}
	\end{equation}
	where $\rho$ denotes the weight parameter and 
	$\mathbf{A}$ contains a wavelet basis corresponding to an inverse wavelet transform.
	Eq.~\eqref{eq:ModelWithProxTerm} provides a balance of $\bm{\alpha}^{k}$ and fidelity term, thus an additional benefit is that the restoration $\bm{\alpha}^{k}$ can be corrected by the fidelity while $\bm{\alpha}^{k}$ is out of the desired descent direction. 
	Benefit of the continuity of the function in Eq.~\eqref{eq:ModelWithProxTerm}, a closed solution can be derived through Eq.~\eqref{eq:closed}:
	\begin{equation}
	\begin{aligned}
	\mathbf{u}^{k+1\!}\! = \!\mathcal{F}\left(\bm{\alpha}^{k};\!\rho\right)\!=\!
	\mathbf{A}^\mathrm{T}\mathbf{F}^\mathrm{T}\!\left(\mathbf{P}^\mathrm{T}\mathbf{P}\!+\!\rho \mathbf{I}\right)^{\!-1\!}
	\left(\mathbf{P}^\mathrm{T}\mathbf{y}\!+\!\rho \mathbf{FA}\bm{\alpha}^k\right).
	\end{aligned}
	\label{eq:closed}
	\end{equation}
	
	\noindent{\textbf{Data-driven Module:}}
	Inspired by the deep methods which can effectively simulate the data distribution by training numerous paired input and output. 
	A learning based module with deep architecture is taken into account to utilize the implicit data based distribution. Notice that previous work has demonstrated an insufficient sparsity representation in the pre-defined sparse transform will introduce artifacts during the process\cite{knoll2011second,liang2011sensitivity}. Furthermore, the previous fidelity based optimization may also bring in artifacts. Thus a residual learning network with shortcut is adopted as a denoiser. We define this module as 
	\begin{equation}
	\mathbf{v}^{k+1} = \mathcal{N}\left(\mathbf{u}^{k+1};\vartheta^{k+1}\right),
	\end{equation}
	where $\vartheta^{k+1}$ is the parameter of network in the $k$-th stage. Notice that here both the input and output are in image domain. Thus the transformation between image and sparse code is incorporated into the network. In the experiments section, we will give a concrete example of the choice for such CNN denoiser. 
	
	\noindent{\textbf{Optimal Condition Module:}}
	Two major arguments arising with data-driven module are: 1) whether the direction generated by deep architecture can satisfy the convergence analysis (i.e., updating along a descent direction); 2) losing the empirical prior, whether the output of network whether trends to search a desired optimality solution. 
	
	To clarify above doubts, a checking mechanism based on first-order optimal condition is proposed to indicate whether the updating by deep architecture is satisfied and which variable should be adopted in the next iterations. First, we introduce a proximal gradient with momentum term to connect the output of data-driven network with the first-order optimal condition of a constructed minimization energy. Here, we define the momentum proximal gradient\footnote{$\mathtt{prox}_{\eta \lambda\|\cdot\|_{p}}(\mathbf{v}) = \arg\min_{\mathbf{x}} \lambda  \|\mathbf{x}\|_{p} +\frac{1}{2} \|\mathbf{x} - \mathbf{v}\|^2$.} as
	\begin{equation}
	\bm{\beta}^{k+1\!}\!\in\!\mathtt{prox}_{\!\eta_1\lambda\|\!\cdot\!\|_p}
	\!\left(\!\mathbf{v}^{k+1\!}\!-
	\!\eta_1\!\left(\!\nabla f\!\left(\!\mathbf{v}^{k+1\!}\right)\!+
	\!\rho\!\left(\!\mathbf{v}^{k+1\!}\!-\!\bm{\alpha}^{k\!}\right)\!\right)\!\right),
	\label{eq:checkingprox}
	\end{equation}
	where $\eta_1$ is the step-size and $f$ denotes the fidelity term in Eq.~\eqref{eq:OriginalModel}. Then we establish feedback mechanism by considering the first-order optimal condition of Eq.~\eqref{eq:checkingprox} as 
	\begin{equation}
	\|\mathbf{v}^{k+1}-\bm{\beta}^{k+1}\| \leq \varepsilon^{k}\|\bm{\alpha}^{k}-\bm{\beta}^{k+1}\|.
	\label{eq:error}
	\end{equation}
	Here, $\varepsilon^{k}$ is a positive constant to reveal the tolerance scale of the distance between current solution $\bm{\beta}^{k+1}$ and the last updating $\bm{\alpha}^{k} $ at the $k$-th stage. Finally, $\bm{\alpha}^{k}$ is also re-considered when Eq.~\eqref{eq:error} is not satisfied. Thus, our checking module can be simplified as
	\begin{equation}
	\begin{array}{l}
	\mathbf{w}^{k+1}= \mathcal{C}( \mathbf{v}^{k+1} ,\bm{\alpha}^{k} ) 
	= \left\{
	\begin{array}{ll}
	\bm{\beta}^{k+1}    & \ \text{Eq}.~\eqref{eq:error}\\
	\bm{\alpha}^{k}     & \ \text{otherwise}.\\
	\end{array} \right.
	\end{array}
	\label{eq:w}
	\end{equation}

	\noindent{\textbf{Prior Module:}}
	Manual sparse prior is a widely used component to constrain the desired solution in traditional optimization. It also naturally describes the distribution of MRI. Thus, it makes sense to introduce the sparsity for a better CS-MRI restoration. We append a prior module after checking mechanism to pick the penalty term $\lambda\|\mathbf{v}\|_p$ in Eq.~\eqref{eq:OriginalModel} up again. Considering the nonconvexity of $\ell_p$-norm regularization, we solve it by a step of proximal gradient as following:
	
	\begin{equation}
	\!\bm{\alpha}^{k+1}\! =\!\mathcal{P}\!\left( \mathbf{w}^{k+1};\!\eta_2\right)\! \in \! \mathtt{prox}_{\eta_2 \lambda\|\cdot\|_p}\!\left(\mathbf{w}^{k+1}\!-\!\eta_2\nabla f\!\left(\mathbf{w}^{k+1}\!\right)\!\right),
	\end{equation}
	where $\eta_2$ is the step-size. In this way, we can enhance the effect of original model and naturally correct over-smooth and preserve more details.
	
	Considering all the aforementioned settlements, we can reconstruct the fully-sampled MR data by iteratively solving corresponding subproblems as showed in Alg.~\ref{alg1}. We restore the clear MRI by the final estimation $\bm{\alpha}^{*}$.

	\begin{algorithm}[!t]
		\caption{The proposed framework} 
		\label{alg1}
		\begin{algorithmic}[1]
			\REQUIRE $ \mathbf{x^0, P, F, y,}$ and some necessary parameters.
			\ENSURE Reconstructed MR image $ \mathbf{x}$. 
			\WHILE{ not converged}
			\STATE $\mathbf{u}^{k+1}  =  \mathcal{F}\left(\bm{\alpha}^{k};\rho\right), $
			\STATE $\mathbf{v}^{k+1}  = \mathcal{N}\left(\mathbf{u}^{k+1};\vartheta^{k+1}\right), $
			\STATE $\mathbf{w}^{k+1} =\mathcal{C}\left( \mathbf{v}^{k+1} , \bm{\alpha}^{k}\right), $
			\STATE $\bm{\alpha}^{k+1} = \mathcal{P}\left( \mathbf{w}^{k+1};\eta_2\right),$
			\ENDWHILE
			\STATE $\mathbf{x} = \mathbf{A}\bm{\alpha}^{*}. $
		\end{algorithmic}
	\end{algorithm}

	\section{Theoretical Investigations}
	Different from existing deep optimization strategies~\cite{sun2016deep,diamond2017unrolled}, which discard the convergence guarantee in iterations and almost rely on the distribution of training data. Our scheme not only merges the traditional designed model optimization about fidelity and sparse prior, but also introduces the learnable network architecture into our deep optimization framework. Furthermore, we also provide a effective mechanism to discriminate whether the output of network in current iteration is a desired descent direction. In the following, we will analysis the strict convergence behavior of our method. 
	
	To simplify the following derivations, we first rewrite the function in Eq.~\eqref{eq:OriginalModel} as
	$$
	\begin{array}{l}
	\Phi (\bm{\alpha}) =\frac{1}{2}\| \mathbf{PF\mathbf{A}\bm{\alpha}-y}\|_2^2+\lambda \|\bm{\alpha}\|_p.
	\end{array}
	$$
	Furthermore, we also state some properties about $\Phi$ which are helpful for the convergence analysis as following\footnote{All the proofs in this paper are presented in \cite{liu2018theoretically}}:
	
	(1) $\frac{1}{2}\| \mathbf{PF\mathbf{A}\bm{\alpha}-y}\|_2^2$ is proper and Lipschitz smooth;
	
	(2) $\lambda \|\bm{\alpha}\|_p$ is proper and lower semi-continuous;
	
	(3) $\Phi (\bm{\alpha})$ satisfies K{\L} property and is coercive.
	
	\noindent Then some important Propositions are given to illustrate the convergence performance of our approach.
	
	\begin{proposition}\label{prop:c-error}
		Let $ \left\{\bm{\alpha}^k\right\}_{k\in\mathbb{N}} $ and  $\left\{\bm{\beta}^k\right\}_{k\in\mathbb{N}} $ be the sequences generated by Alg.~\ref{alg1}. Suppose that the error condition 
		$\|\mathbf{v}^{k+1}-\bm{\alpha}^{k}\| \leq \varepsilon^{k}\|\bm{\beta}^{k+1}-\bm{\alpha}^{k}\|$ 
		in our $\mathtt{icheck}$ is satisfied. Then there existed a sequence $\{C^{k}\}_{k\in\mathbb{N}}$ such that 
		\begin{equation}	
		\Phi(\bm{\beta}^{k+1}) \leq \Phi(\bm{\alpha}^k)-C^{k}\|\bm{\beta}^{k+1}-\bm{\alpha}^k\|^2,
		\end{equation}
		where $C^{k} = {1}/{2\eta_{1}} - {L_f}/{2} -  (L_f  + |\rho-{1}/{\eta_{1}}|)\epsilon^{k} >0$ and $L_f$ is the Lipschitz coefficient of $\nabla f$ .
		\label{eq:ineq_fun_pgmomentum}
	\end{proposition}
	
	\begin{proposition}\label{prop:pg}
		If $\eta_2 < 1/L_f$, let $ \left\{\bm{\alpha}^k\right\}_{k\in\mathbb{N}} $ and $\left\{\mathbf{w}^k\right\}_{k\in\mathbb{N}}$ be the sequences generated by a proximal operator in Alg.~\ref{alg1}. Then we have
		\begin{equation}
		\Phi(\bm{\alpha}^{k+1}) \leq \Phi(\mathbf{w}^{k+1})-({1}/({2\eta_2}) - {L_f}/{2})\|\bm{\alpha}^{k+1}-\mathbf{w}^{k+1}\|^2.\label{eq:ineq_fun_pg}
		\end{equation}	
	\end{proposition}
	\begin{remark}
		Actually, the inequality in Propositions \ref{prop:c-error} and \ref{prop:pg} provide a descent sequence $\Phi(\bm{\alpha}^{k})$ by illustrating the relationship of $\Phi(\bm{\alpha}^{k})$ and $\Phi(\mathbf{w}^k)$ as
		$$
		-\infty <\Phi(\bm{\alpha}^{k+1}) \leq \Phi(\mathbf{w}^k) \leq \Phi(\bm{\alpha}^k) \leq \Phi(\mathbf{\alpha}^0).
		$$
		Thus, the operator $\mathcal{C}$ is a key criterion to check the output of our designed network whether propagates along a descent direction. Moreover, it also ingeniously builds a bridge to connect the adjacent iteration $\Phi(\bm{\alpha}^{k})$ and $\Phi(\bm{\alpha}^{k+1})$.
	\end{remark}
	
	\begin{theorem}\label{theorem:convergence}
		Suppose that $ \left\{\bm{\alpha}^k\right\}_{k\in\mathbb{N}} $ be a sequence generated by Alg.~\ref{alg1}. The following assertions hold.
		\begin{itemize}
			\item The square summable of sequence $\left\{\bm{\alpha}^{k+1}-\mathbf{w}^{k+1} \right\}_{k\in\mathbb{N}}$ is bounded, i.e., 
			$
			\sum_{k=1}^{\infty}\|\bm{\alpha}^{k+1}-\mathbf{w}^{k+1}\|^2 < \infty.
			$
			\item The sequence $ \left\{\bm{\alpha}^k\right\}_{k\in\mathbb{N}} $ converges to a critical point $\bm{\alpha}^{*}$ of $\Phi$.
		\end{itemize}
	\end{theorem}
	
	\begin{remark}
		The second assertion in Theorem~\ref{theorem:convergence} implies that $\left\{\bm{\alpha}^k\right\}_{k\in\mathbb{N}}$ is a Cauchy sequence, thus it globally converges to the critical point of $\Phi$. 
	\end{remark}

	\section{Real-world CS-MRI with Rician Noises}
	It is worth noting that robustness is important in CS-MRI. Unfortunately, most strategies only consider the scenario without noise thus they usually fail on real-world cases. To enable our method handle the task of CS-MRI with Rician noise, we extend our CS-MRI model in Eq.~\eqref{eq:OriginalModel} as:
	\begin{equation}
	\begin{array}{l}
	\min\limits_{\mathbf{x}}\ 
	\frac{1}{2}\| \mathbf{PF}\mathbf{z}-
	\mathbf{y}\|_2^2+\lambda_1\|\mathbf{A}_1^{\top}\mathbf{x}\|_p +
	\lambda_2\|\mathbf{A}_2^{\top}\mathbf{z}\|_p  \\
	s.t.\ \ \ \ \mathbf{z}=\sqrt{(\mathbf{x}+\mathbf{n_1})^2+\mathbf{n}_2^2},
	\end{array}
	\label{eq:rician1}
	\end{equation}
	where $\mathbf{x}$, $\mathbf{y}$, $\mathbf{n}_1$, and $\mathbf{n}_2$ denotes the fully sampled clear MR image, sparse k-space data, the uncorrelated Gaussian noise in real and imaginary component respectively. $\mathbf{A}_1^{\top}$ and $\mathbf{A}_2^{\top}$ are inverse of the wavelet transform $\mathbf{A}_1$ and $\mathbf{A}_2$ respectively.
	
	\noindent{\textbf{Optimization Strategy:}}
	We can rewrite Eq.~\eqref{eq:rician1} by splitting it as the following subproblems:
	\begin{subequations}
		\begin{numcases}{}
		\begin{array}{l}
		\!\!\!\!\mathbf{z}^{k+1} \! =\!\arg\min\limits_{\mathbf{z}}
		\frac{1}{2}\| \mathbf{PFz}-\mathbf{y}\|_2^2 +
		\lambda_1\|\mathbf{A}_1^{\top}\mathbf{z}\|_p\\
		\qquad \quad +\frac{\rho_1}{2}\|\mathbf{z}-  \sqrt{(\mathbf{x}^{k}+\mathbf{n}_1)^2+
			\mathbf{n}_2^2}  \|_2^2,
		\end{array}\label{eq:sub_z}\\
		\begin{array}{l}
		\!\!\!\!\mathbf{x}^{k+1}\!  =\!\arg\min\limits_{\mathbf{x}}
		\frac{1}{2}\|\sqrt{(\mathbf{x}\!\!+\!\!\mathbf{n}_1)^2 +
			\mathbf{n}_2^2}-\mathbf{z}^{k+1}\|_2^2\\
		\qquad \quad+\lambda_2\|\mathbf{A}_2^{\top}\mathbf{x}\|_p,
		\end{array}\label{eq:sub_x}
		\end{numcases}
	\end{subequations}
	where $\rho_1$ is the penalty parameter.
	Thus, we can subsequently tackle each of them with the iterations in Alg.~\ref{alg1} to solve $\mathbf{z}^{k+1}$ and $\mathbf{x}^{k+1}$ respectively.

	For the subproblem Eq.~\eqref{eq:sub_z}, which aims to reconstruct the fully sampled noisy MR image from k-space observation $\mathbf{y}$, we can optimize it to be similar with the aforementioned CS-MRI problem by assuming the fidelity including both CS and noise (i.,e., two square term), respectively. In terms of the second subproblem Eq.~\eqref{eq:sub_x}, the solution is a clear image. We cannot straightly obtain the closed-form solution since the quadratic term $ (\mathbf{x}\!+\!\mathbf{n_1})^2 $. Thus, a learnable strategy is adopted to restore a rough MRI in fidelity module.

	\begin{figure}[!tbp]
		\begin{center}
			\begin{tabular}{c@{\extracolsep{-1em}}c@{\extracolsep{0.2em}}c}
				&\includegraphics[width=.24\textwidth]{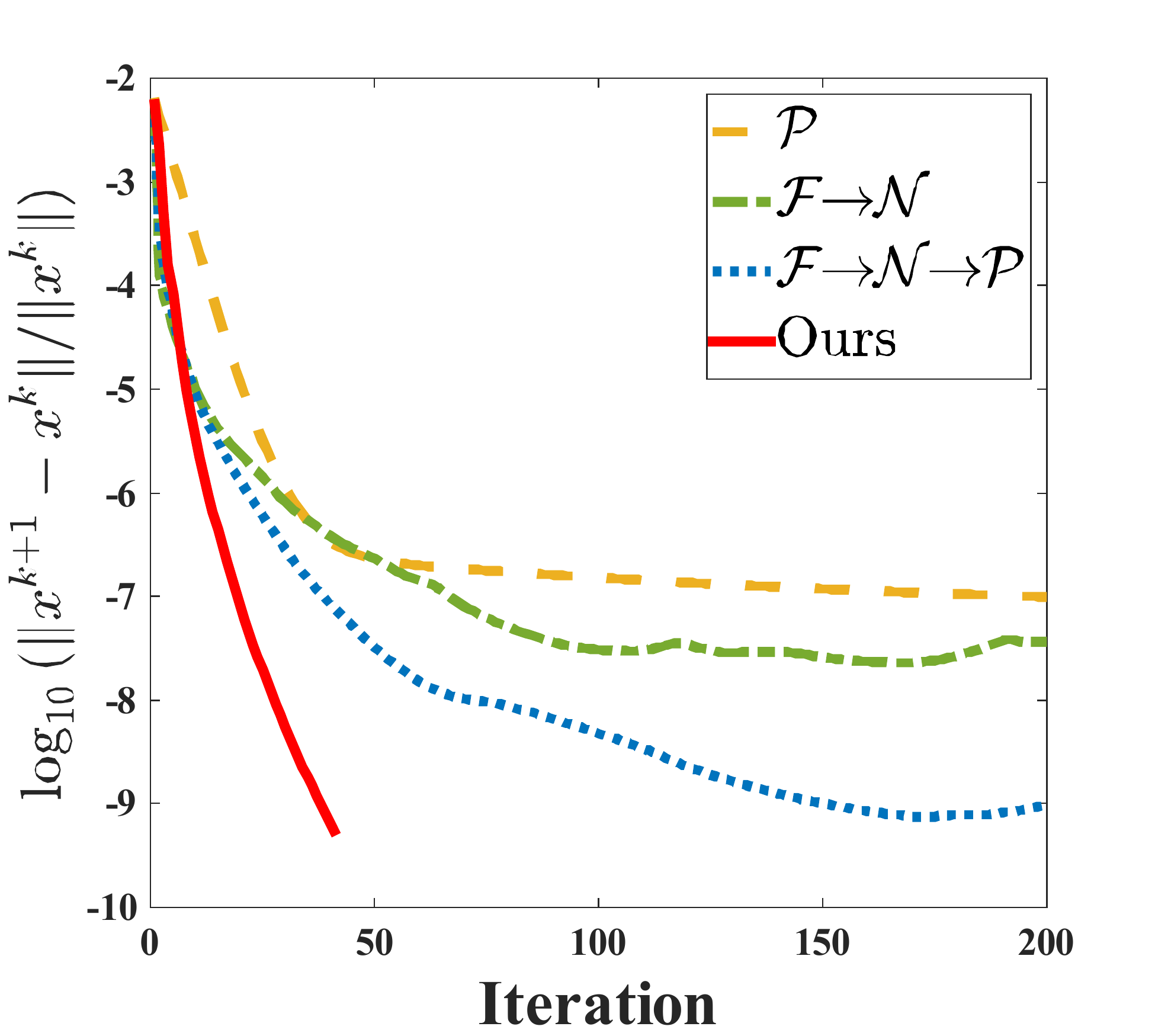}
				&\includegraphics[width=.235\textwidth]{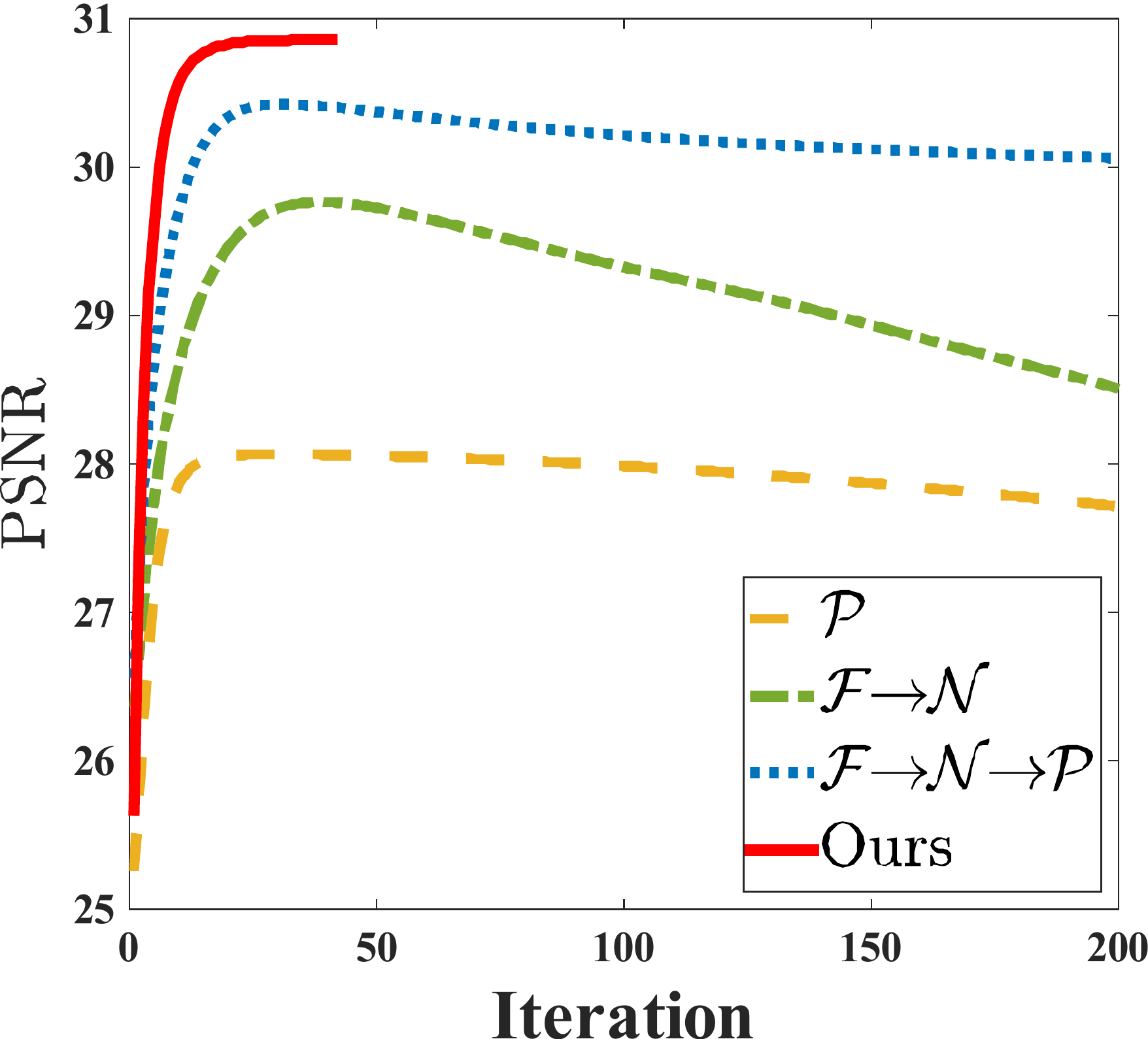}\\
			\end{tabular}
		\end{center}
		\caption{Quantitative results of four optimization models.}
		\label{iteration}
	\end{figure}
	\begin{figure}[!tbp]
		\begin{center}
			\begin{tabular}{c@{\extracolsep{-0.5em}}c@{\extracolsep{0.3em}}c@{\extracolsep{0.3em}}c@{\extracolsep{0.3em}}c@{\extracolsep{0.0em}}c}
				&\includegraphics[width=.11\textwidth]{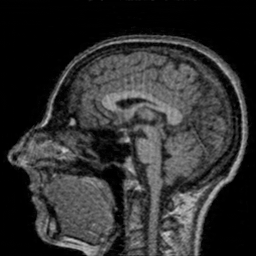}
				&\includegraphics[width=.11\textwidth]{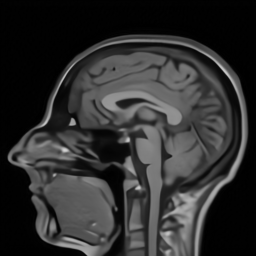}
				&\includegraphics[width=.11\textwidth]{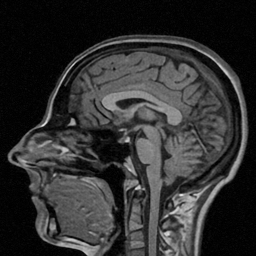}
				&\includegraphics[width=.11\textwidth]{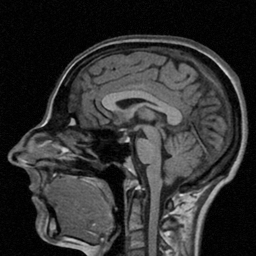}\\
				& $\mathcal{P} $&$\mathcal{F}\!\!\to \!\!\mathcal{N}$& $\mathcal{F}\!\!\to\!\!\mathcal{N}\!\!\to\!\!\mathcal{P} $ & Ours\\
			\end{tabular}
		\end{center}
		\caption{Qualitative results of four optimization models.}
		\label{component}
	\end{figure}
	\begin{figure}[!tb]
		\begin{center}
			\begin{tabular}{l@{\extracolsep{-1.2em}}l@{\extracolsep{4em}}l}
				
				&\includegraphics[width=.48\textwidth]{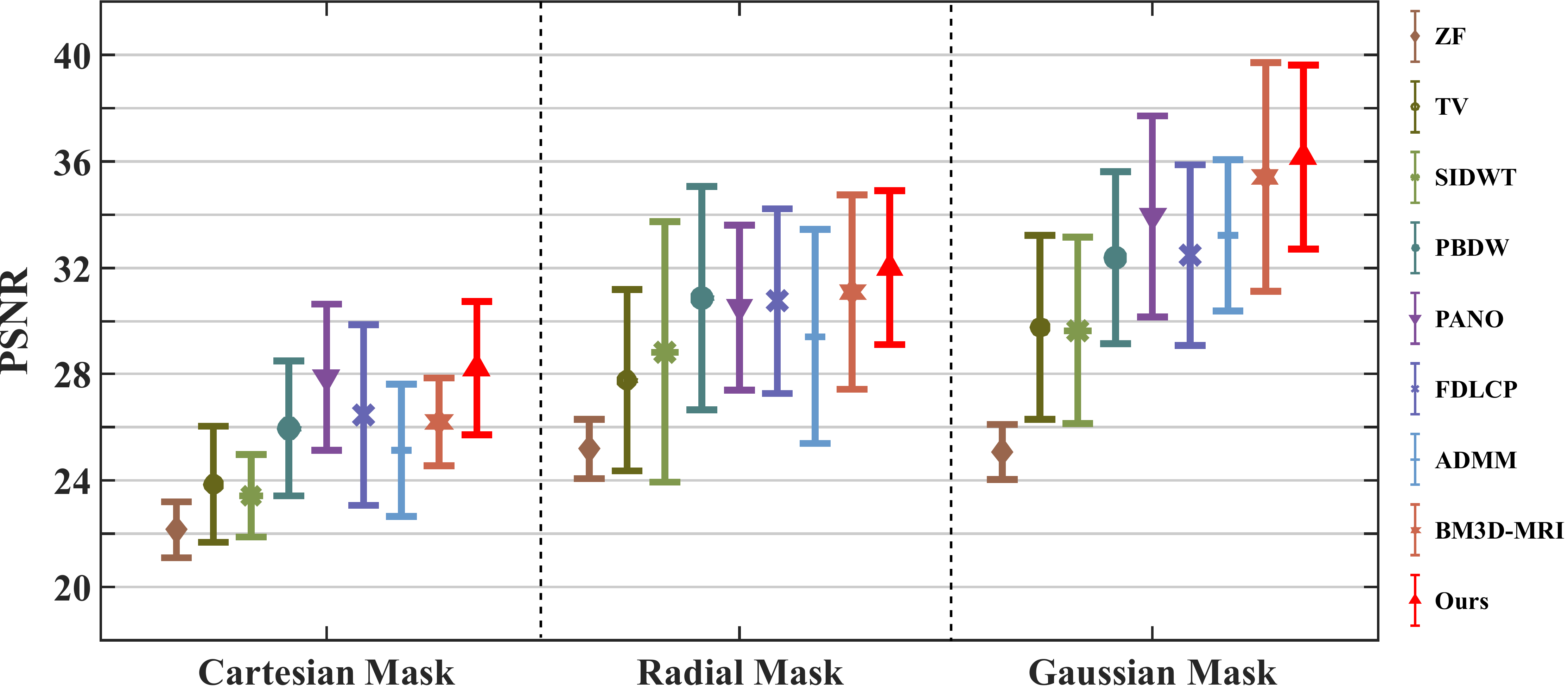}\\
			\end{tabular}
		\end{center}
		\caption{Comparisons using various sampling patterns.}
		\label{errorbar}
	\end{figure}
	\begin{figure}[!tb]
		\begin{center}
			\begin{tabular}{c@{\extracolsep{-1em}}c@{\extracolsep{0.8em}}c}
				&\includegraphics[width=4cm,height=4.0cm]{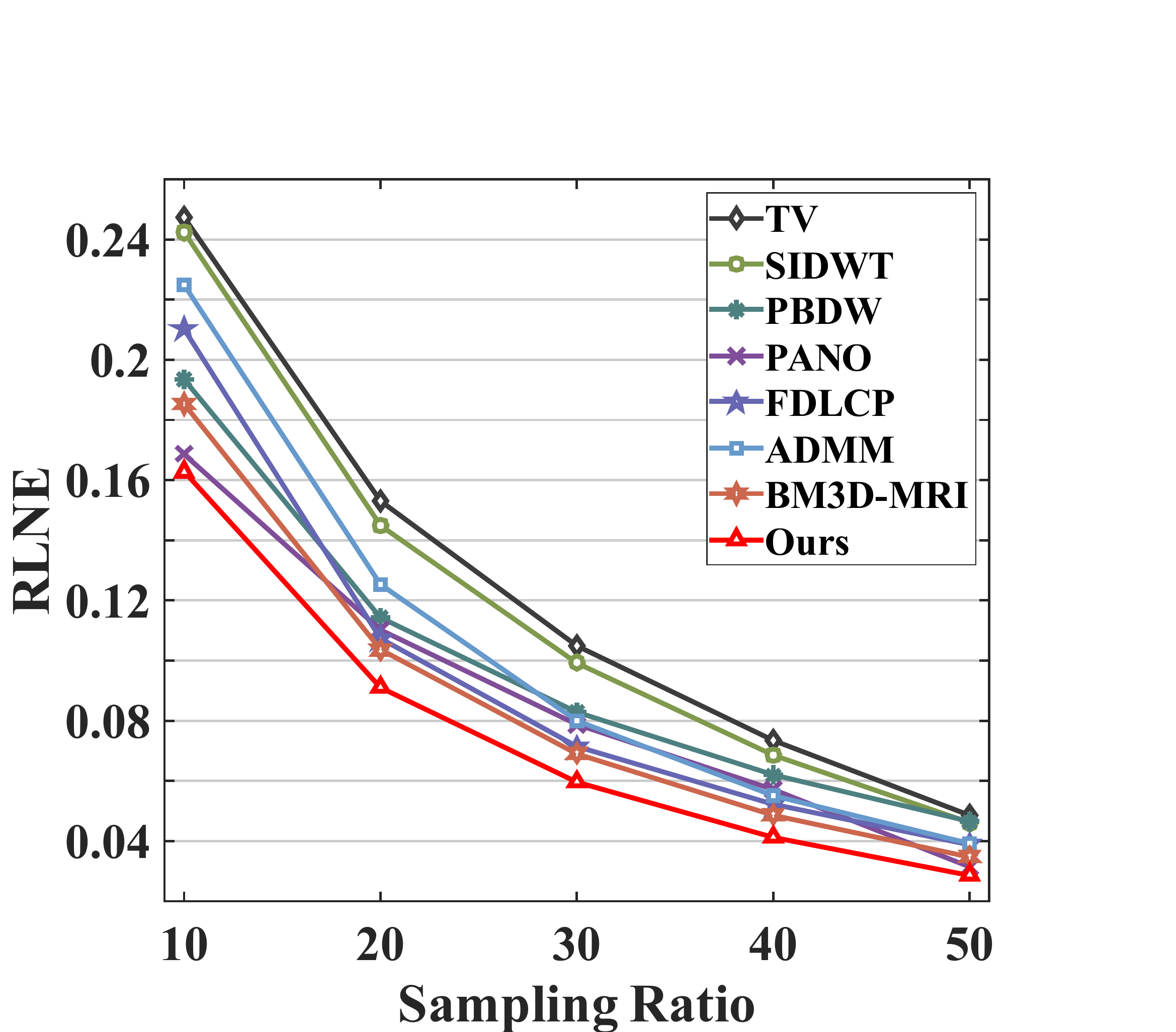}
				&\includegraphics[width=4cm,height=4.06cm]{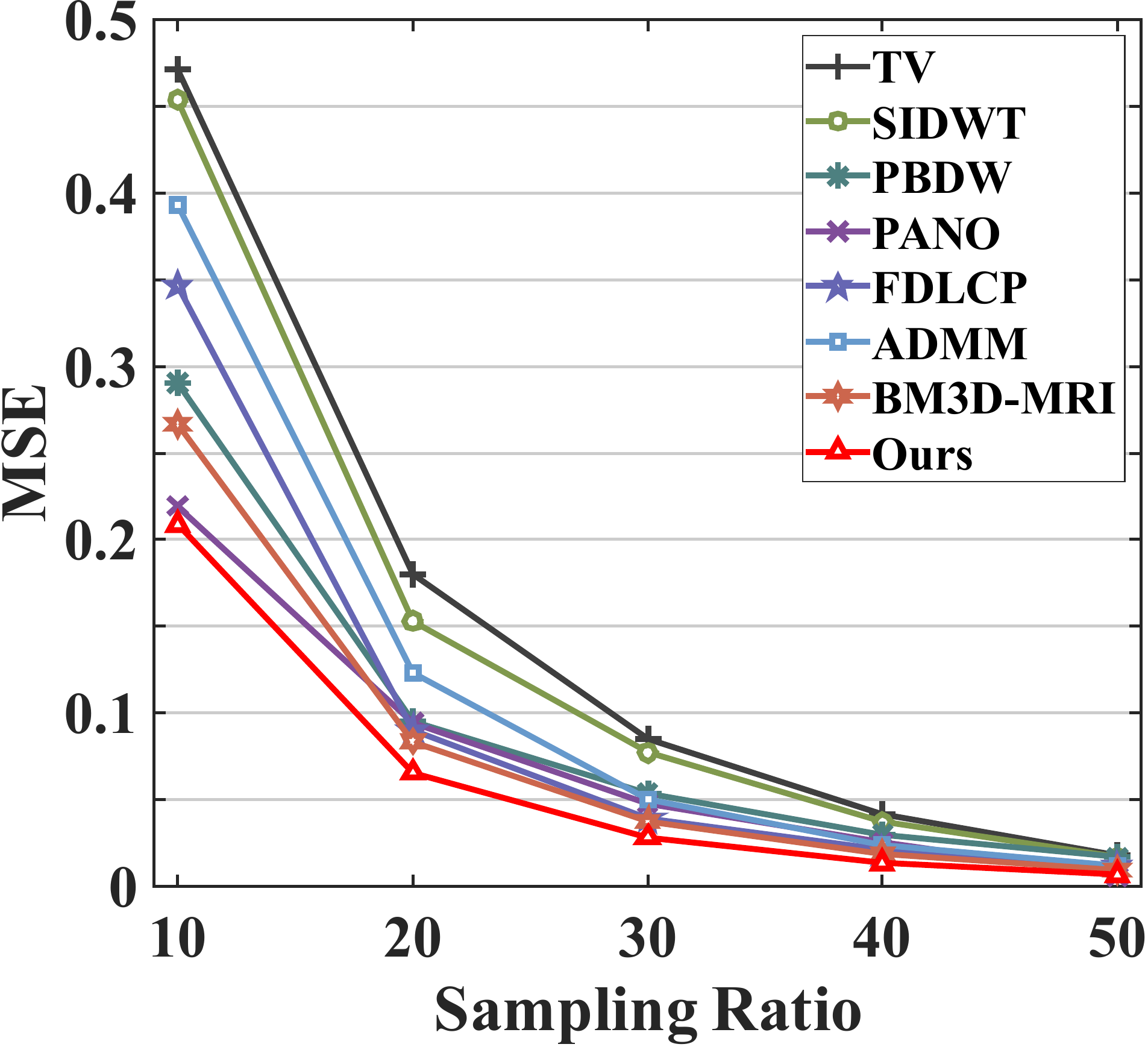}\\
			\end{tabular}
		\end{center}
		\caption{Comparisons under different sampling ratios.}
		\label{error}
	\end{figure}
	\begin{table*}[t]
		\centering
		\caption{Comparison of time consuming using Radial mask with two different sampling ratios (s).}
		\begin{tabular}{c@{\extracolsep{1em}}c@{\extracolsep{0.8em}}c@{\extracolsep{0.8em}}c@{\extracolsep{0.8em}}c@{\extracolsep{0.8em}}c@{\extracolsep{0.8em}}c@{\extracolsep{0.8em}}
				c@{\extracolsep{0.8em}}c@{\extracolsep{0.8em}}c@{\extracolsep{0.8em}}c} 
			\toprule
			Sampling Rate & Zero-filling & TV & SIDWT & PBDW & PANO & FDLCP & ADMM-Net & BM3D-MRI & Ours \\
			\midrule
			10\% & \textbf{0.0035} & 2.0098 & 14.8269 & 45.2639 & 27.9626&  84.8494&1.2332  & 10.8877 &  0.3167 \\
			
			50\% & \textbf{0.0032} & 0.7631 & 5.57611  & 26.1274 & 11.1466&  83.0681 & 1.2499 & 11.5394&  0.5633 \\
			\bottomrule
		\end{tabular}
		\label{time}
	\end{table*}
	\begin{figure*}[!htp]
		\begin{center}
			\begin{tabular}{@{\extracolsep{0.1em}}c@{\extracolsep{0.1em}}c@{\extracolsep{0.1em}}c@{\extracolsep{0.1em}}c@{\extracolsep{0.1em}}c@{\extracolsep{0.1em}}c@{\extracolsep{0.1em}}c@{\extracolsep{0.1em}}c@{\extracolsep{0.1em}}c@{\extracolsep{0.1em}}c}
				&\includegraphics[width=.1\textwidth]{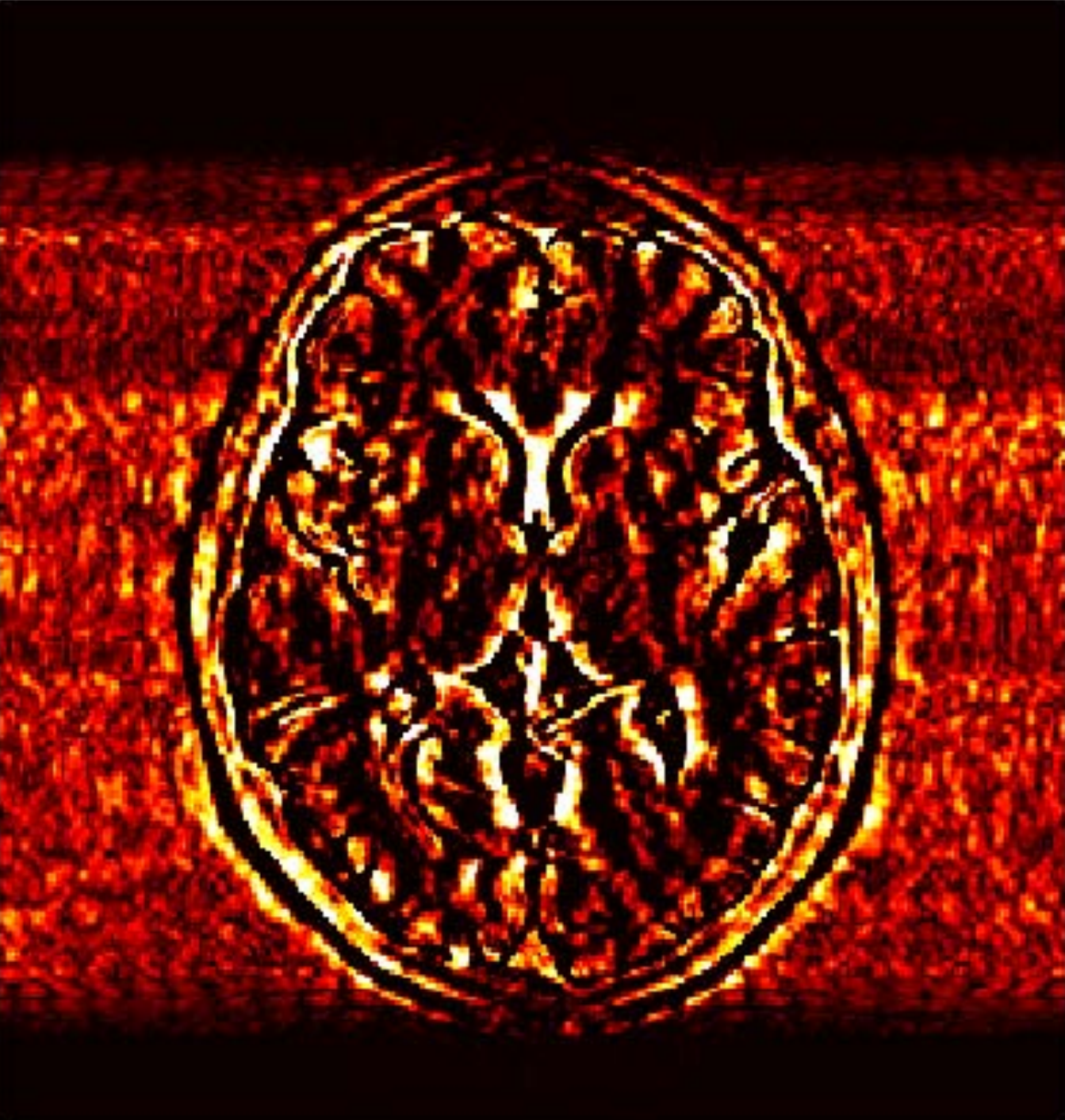}
				&\includegraphics[width=.1\textwidth]{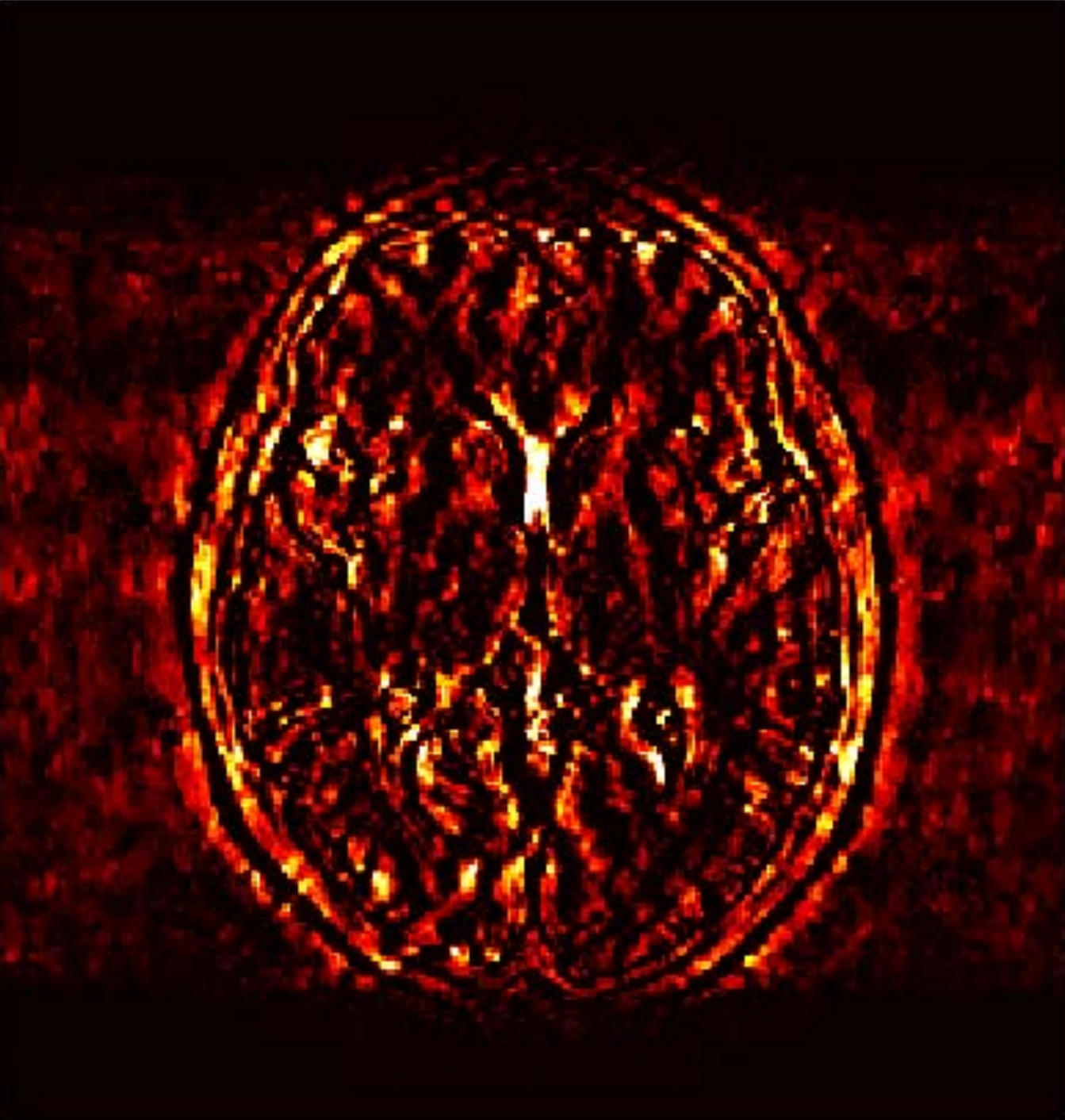}
				&\includegraphics[width=.1\textwidth]{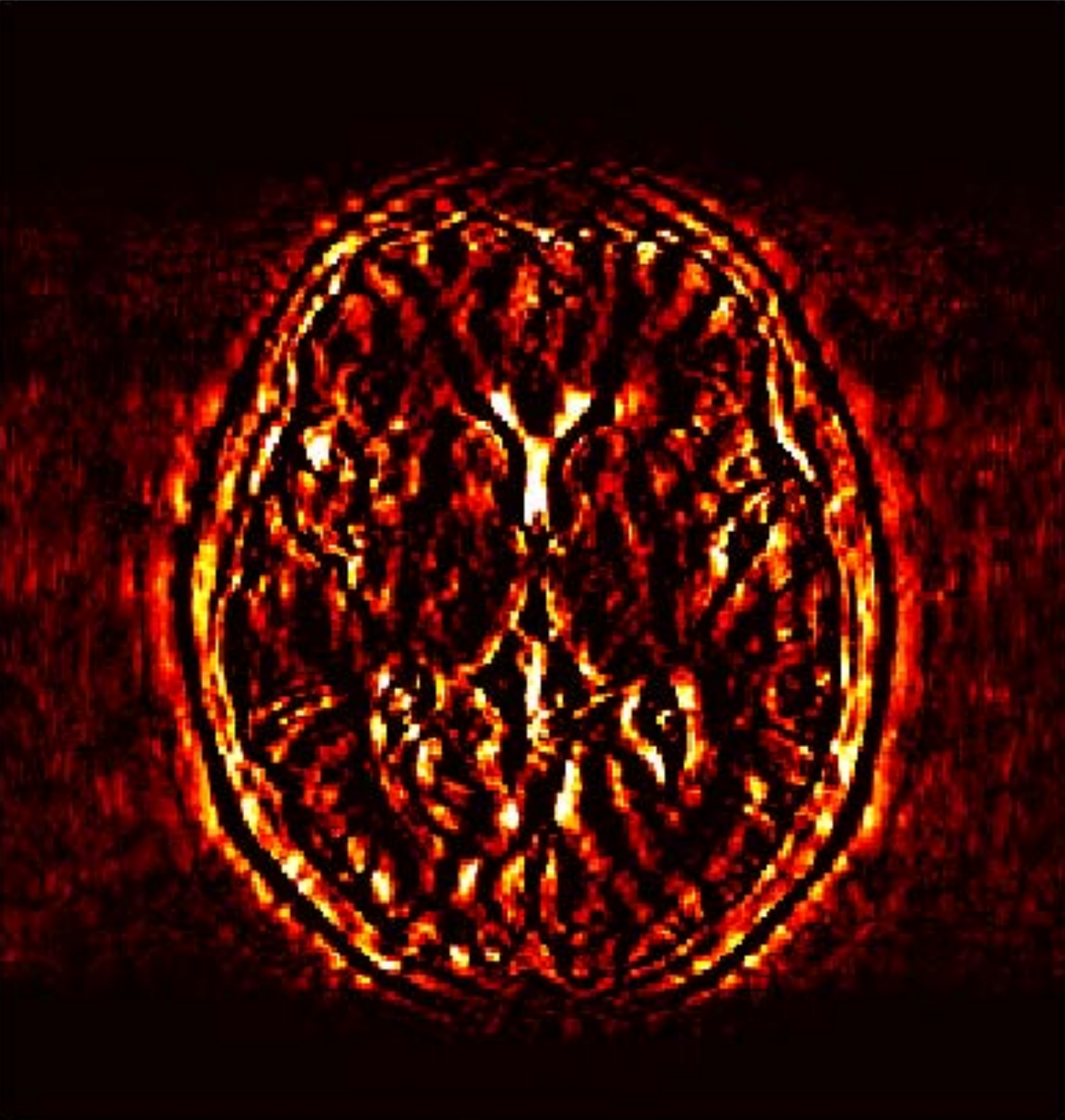}
				&\includegraphics[width=.1\textwidth]{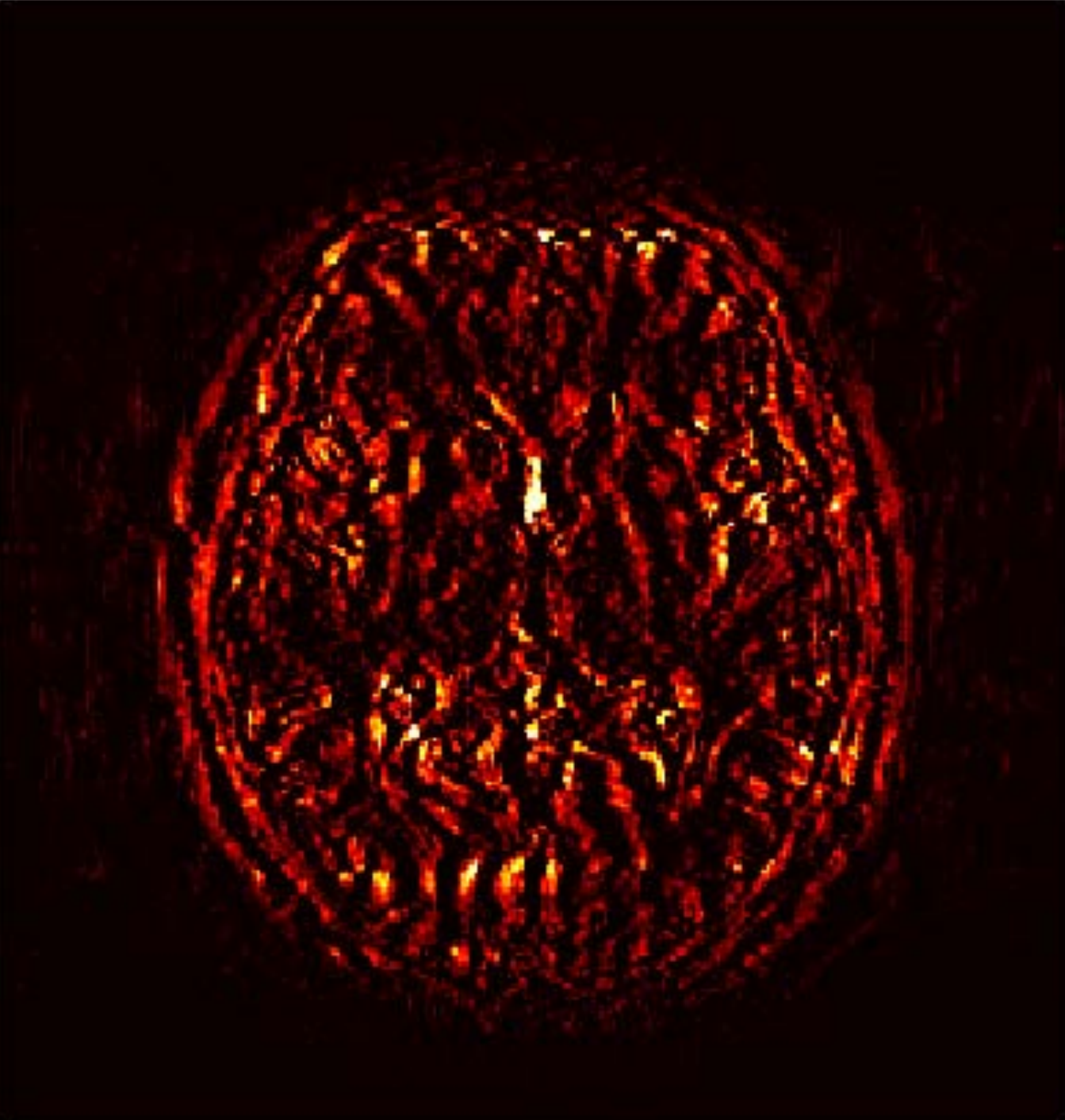}
				&\includegraphics[width=.1\textwidth]{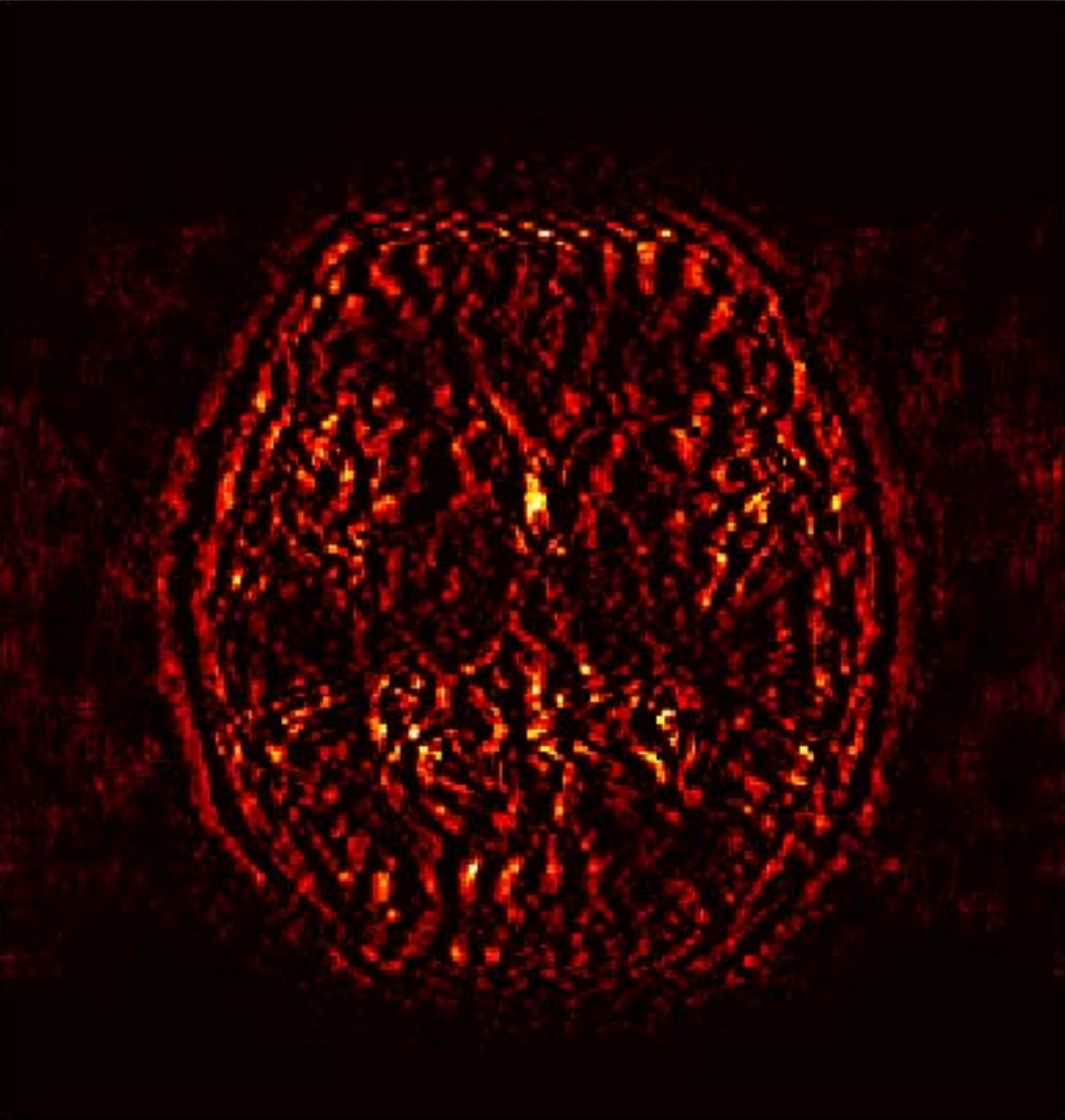}
				&\includegraphics[width=.1\textwidth]{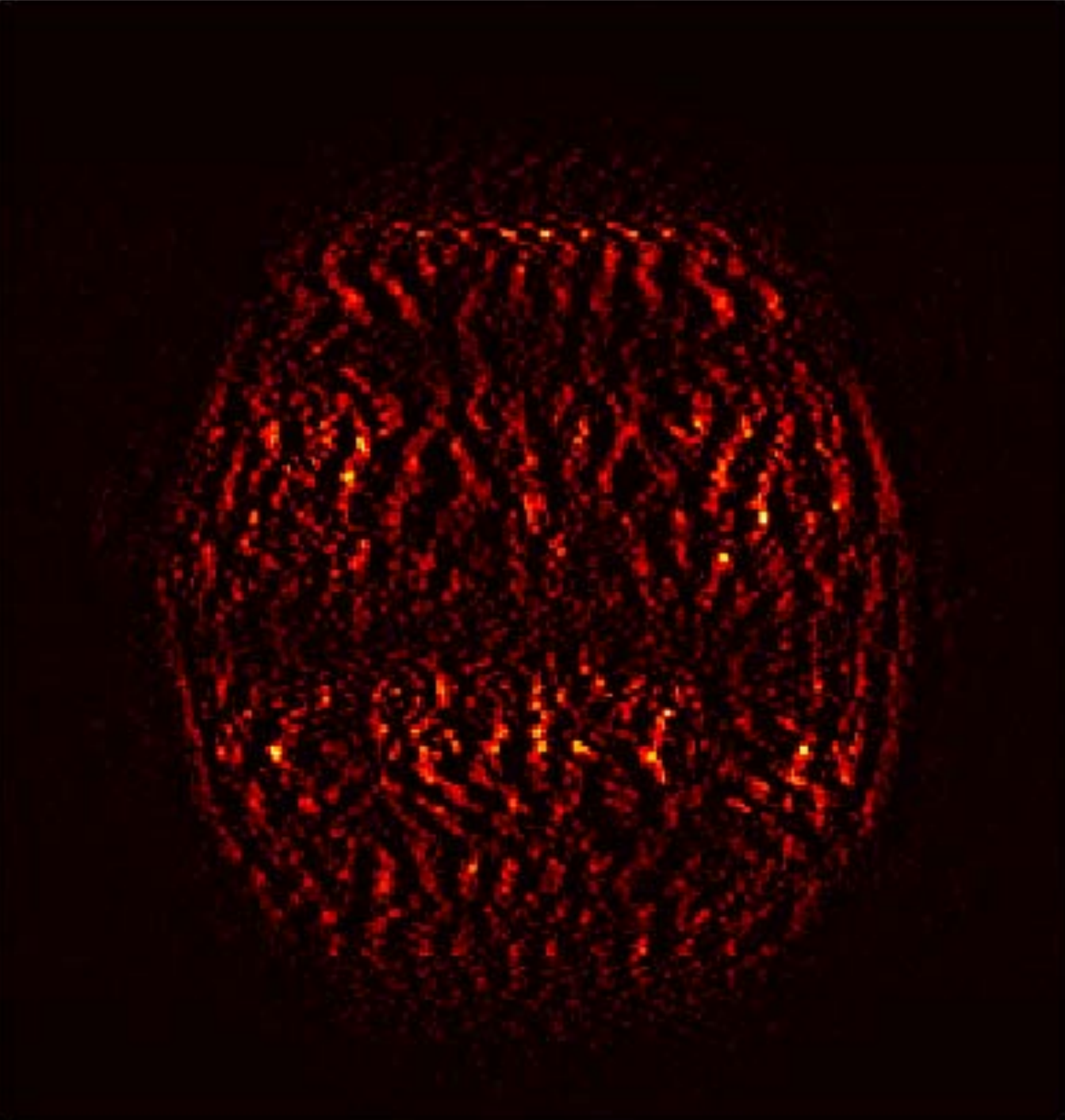}
				&\includegraphics[width=.1\textwidth]{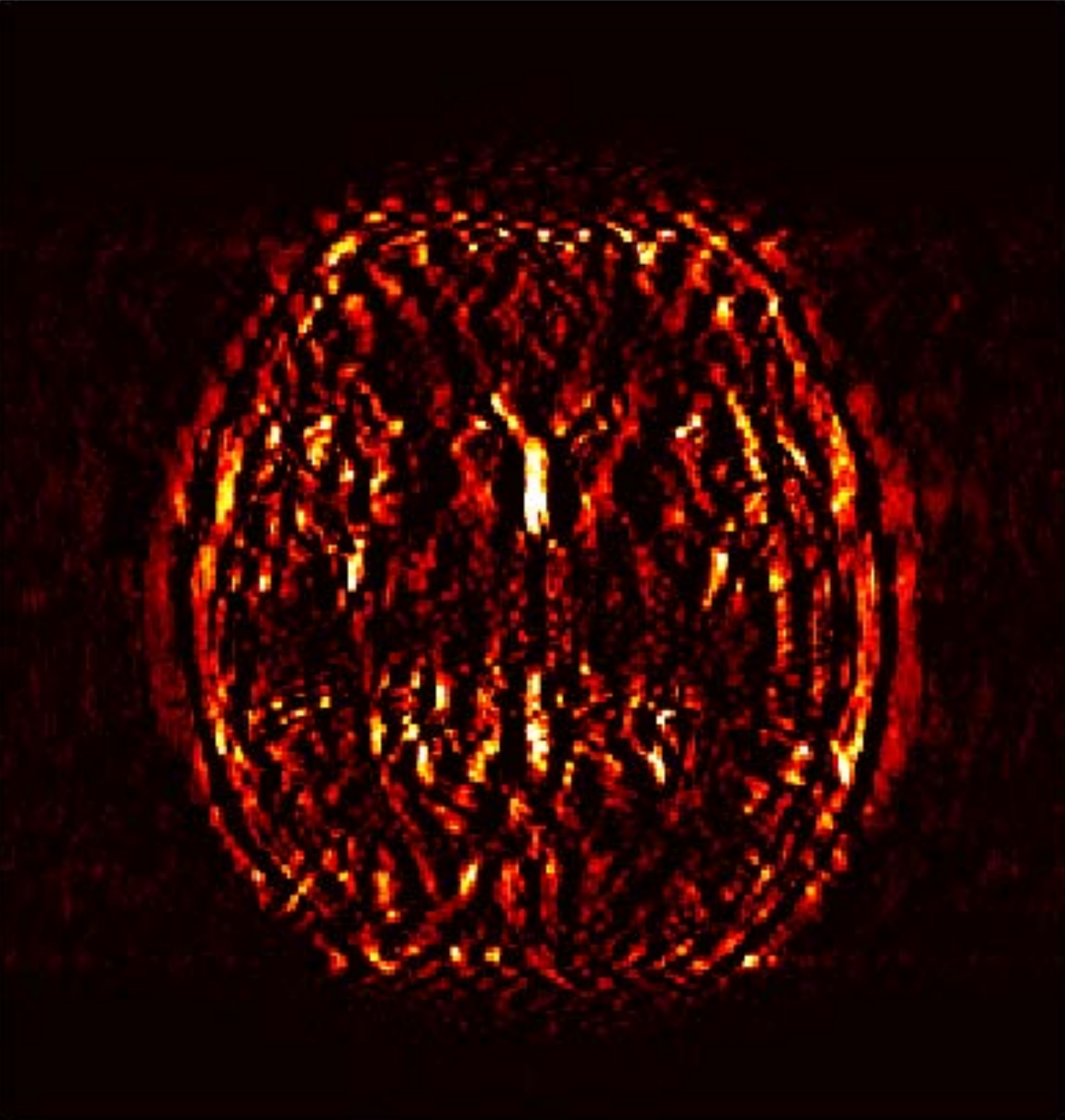}
				&\includegraphics[width=.1\textwidth]{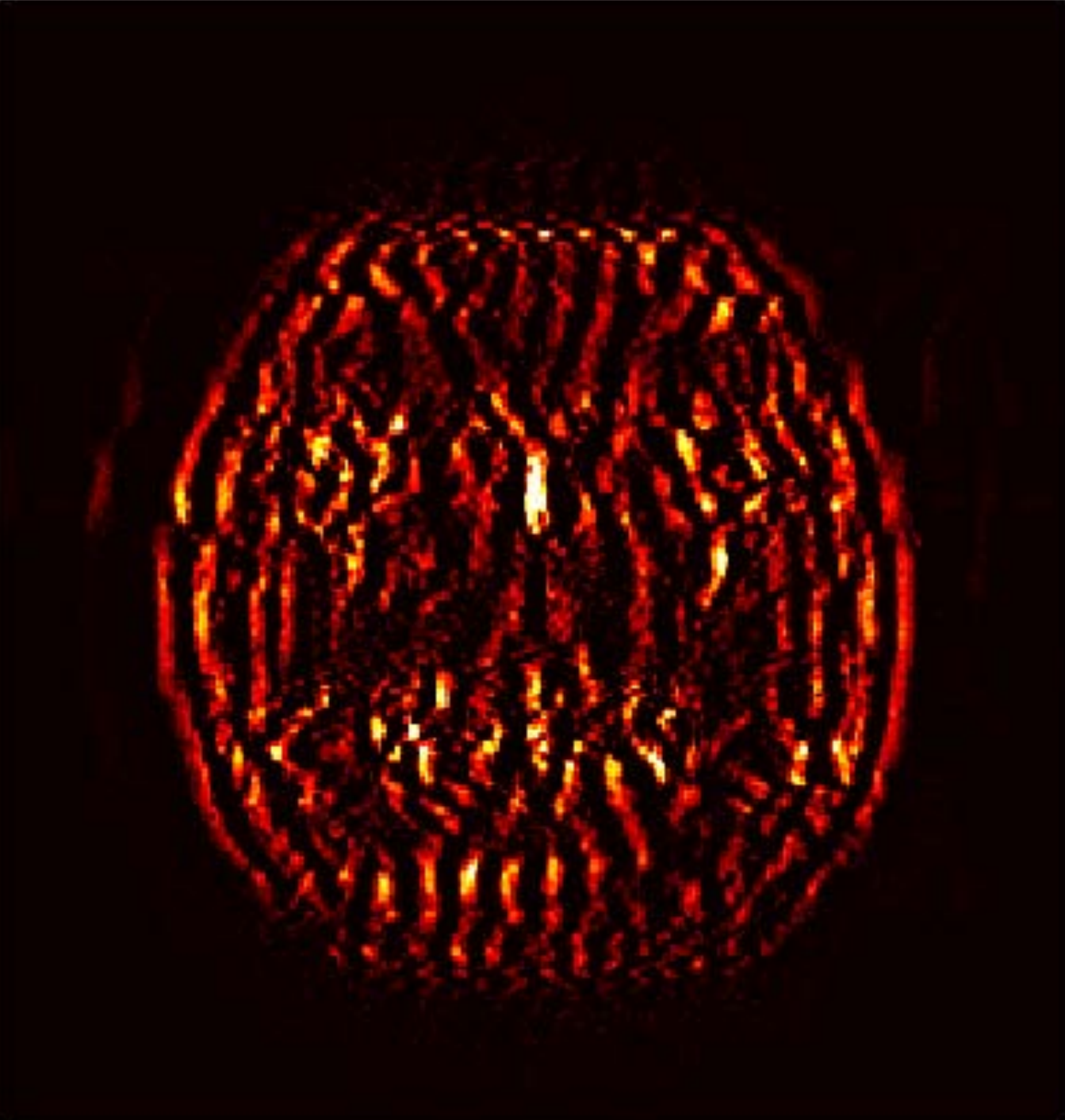}
				&\includegraphics[width=.1\textwidth]{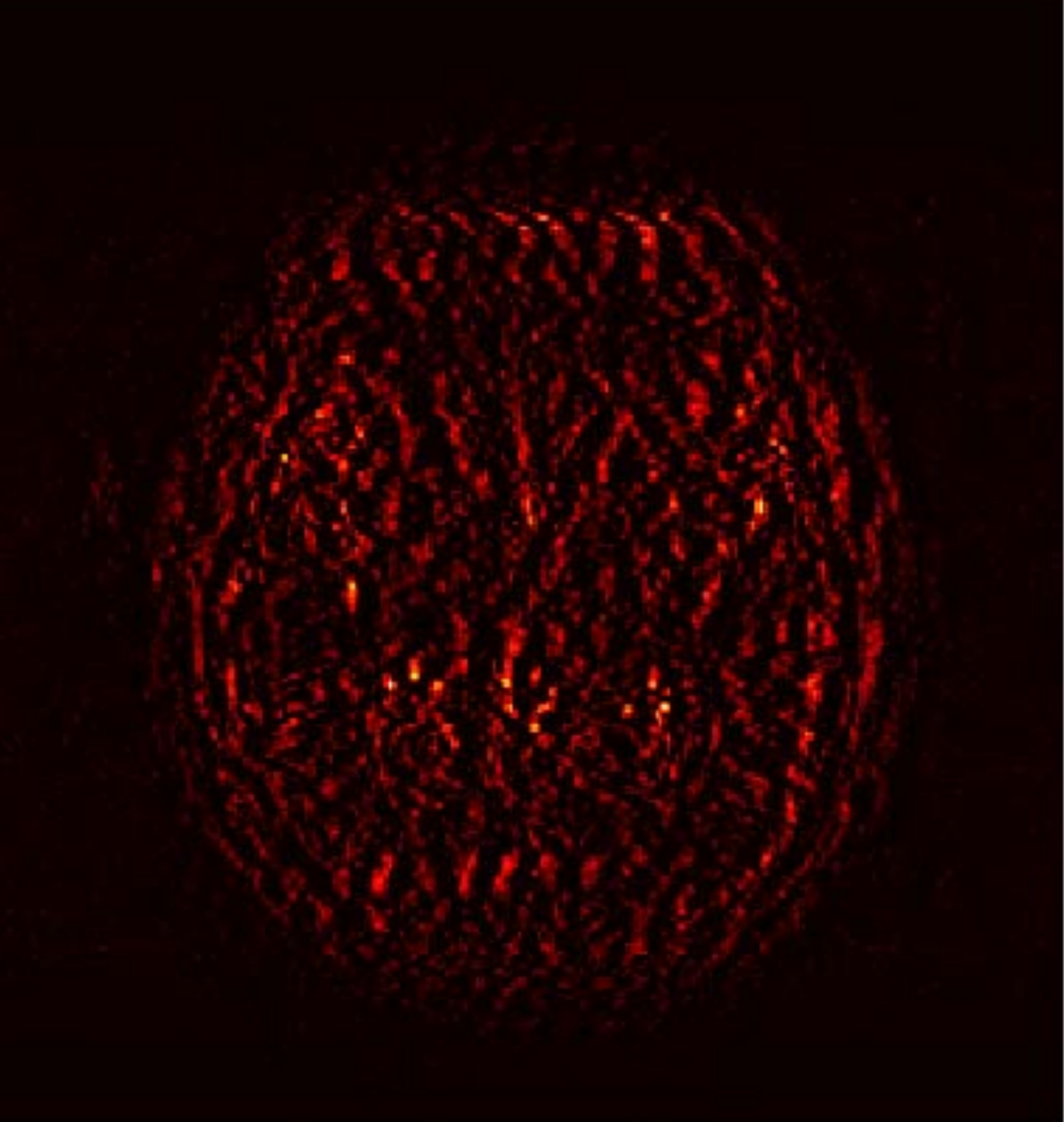}\\
				& Zero-filling	&TV &SIDWT &PBDW &PANO &FDLCP  &ADMM-Net  &BM3D-MRI  &Ours\\	
			\end{tabular}
		\end{center}
		\caption{Visualization of the reconstruction error using Cartesian mask with 30\% sampling rate.}
		\label{heatmap}
	\end{figure*}
	\begin{figure*}[!htp]
		\begin{center}
			\begin{tabular}{l@{\extracolsep{-0.3em}}c@{\extracolsep{0.2em}}
					c@{\extracolsep{0.2em}}c@{\extracolsep{0.2em}}c@{\extracolsep{0.2em}}c}
				&\includegraphics[width=.18\textwidth]{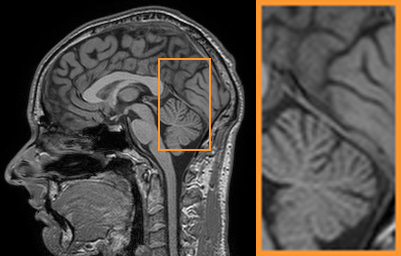}
				&\includegraphics[width=.18\textwidth]{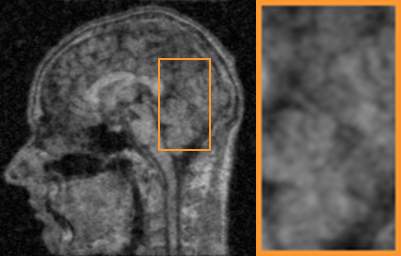}
				&\includegraphics[width=.18\textwidth]{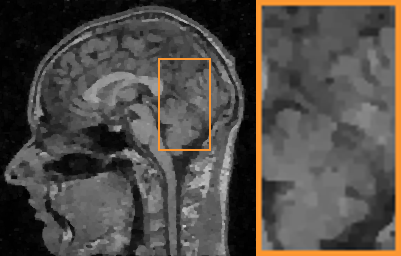}
				&\includegraphics[width=.18\textwidth]{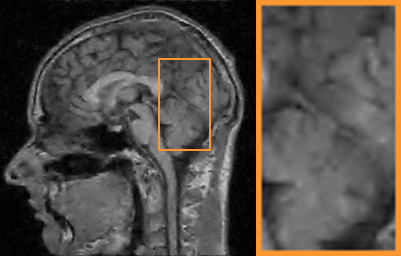}
				&\includegraphics[width=.18\textwidth]{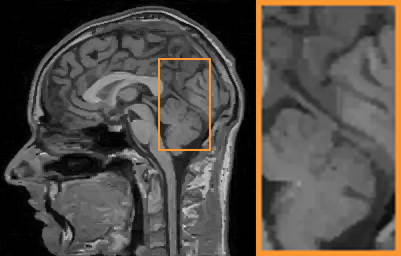}\\
				&Ground Truth (PSNR) & Zero-filling (22.33) & TV (25.22) & SIDWT (25.10) & PBDW (27.39)\\
				&\includegraphics[width=.18\textwidth]{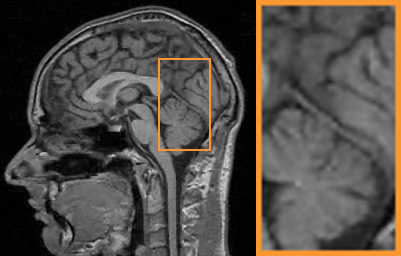}
				&\includegraphics[width=.18\textwidth]{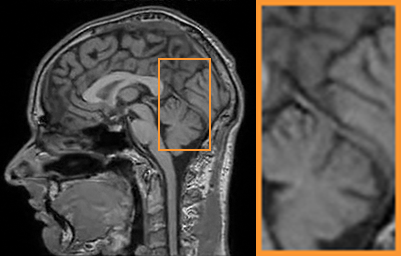}
				&\includegraphics[width=.18\textwidth]{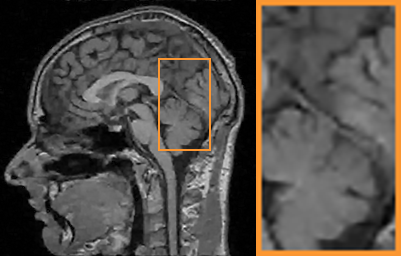}
				&\includegraphics[width=.18\textwidth]{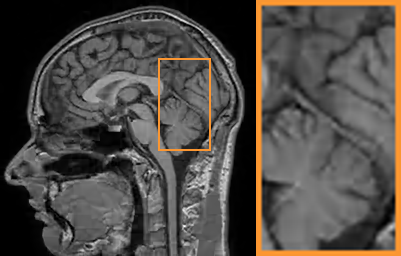}
				&\includegraphics[width=.18\textwidth]{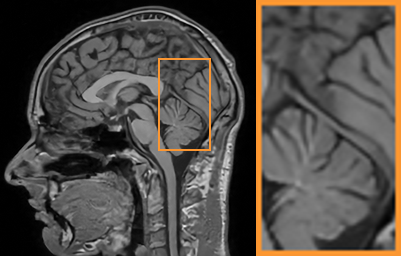}\\
				&PANO (28.77) & FDLCP (29.78) & ADMM-Net (27.91) & BM3D-MRI (29.35) & Ours (\textbf{30.48})\\			
			\end{tabular}
		\end{center}
		\caption{Qualitative comparison on $T_1$-weighted brain MRI data using Gaussian mask with 10\% sample rate.}
		\label{detail1}
	\end{figure*}
	\noindent{\textbf{Learnable Architecture for Rician Noise Removal:}}
	It is tricky to learn the fidelity module by a residual network for Rician noise is not additive. To address this trouble, we design two stage networks based on the relationship of $\mathbf{x}$  and $\mathbf{z}$ (i.e., $\mathbf{x}_{c} $ and $\mathbf{x}_{n} $ in Eq.~\eqref{eq:Rician}). 
	First, we design a residual network to learn $(\mathbf{x}_{c} + \mathbf{n}_{1})^{2}$ from $\mathbf{x}_{n}^{2}$ (i.e., $(\mathbf{x}_{c} + \mathbf{n}_{1})^{2} + \mathbf{n}_{2}^{2}$). To release the square operator under $\mathbf{x}_{c}$, we train the other one by feeding $\sqrt{(\mathbf{x}_{c} + \mathbf{n}_{1})^{2} }$ and $\mathbf{x}_{c}$ as input and output.

	\section{Experimental Results}
	In this section, we first explore the roles of each module and theoretical
	results in our paradigm. To demonstrate the superiority of our method, we then compare it with some state-of-the-art techniques on both traditional and real-world CS-MRI. All experiments are executed on a PC with Intel(R) Gold 6154 CPU @ 3.00GHz 256 GB RAM and a NVIDIA TITAN Xp. Notice that we perform $p=0.8$ in experiments.

	\subsection{CS-MRI Reconstruction}
	We first analyze the effects of modules and verify the theoretical convergence by ablation experiments. Then perform comparisons on traditional CS-MRI in perspective of reconstruction accuracy, time consuming and robustness.
	\begin{table*}[!t]
		\centering
		\caption{Comparison on different testing data using Cartesian mask at a sampling rate of 30\%}
		\begin{tabular}{c@{\extracolsep{1.4em}}c@{\extracolsep{1.4em}}c@{\extracolsep{1.4em}}c@{\extracolsep{1.4em}}c@{\extracolsep{1.4em}}
				c@{\extracolsep{1.4em}}c@{\extracolsep{1.4em}}c@{\extracolsep{1.4em}}c@{\extracolsep{1.4em}}c} 
			\toprule
			MRI Data  & Zero-filling & TV & SIDWT & PBDW & PANO & FDLCP & ADMM-Net & BM3D-MRI & Ours \\
			\midrule
			Chest  	&22.95 	&24.43	&24.17 	&25.91	&27.73 	&26.84 	&25.38 	&26.27 	&\textbf{28.22}\\
			Cardiac &23.40	&29.17	&27.49	&31.34	&33.14	&33.84	&31.42	&31.44	&\textbf{35.69}\\
			Renal 	&24.32	&28.77	&27.66	&31.05	&32.21	&33.74	&31.13	&31.01	&\textbf{34.79}\\	
			\bottomrule
		\end{tabular}
		\label{testingdata}
	\end{table*}
	\begin{figure*}[!t]
		\begin{center}
			\begin{tabular}{l@{\extracolsep{-0.5em}}c@{\extracolsep{0.2em}}
					c@{\extracolsep{0.2em}}c@{\extracolsep{0.2em}}c@{\extracolsep{0.2em}}c}
				&\includegraphics[width=.18\textwidth]{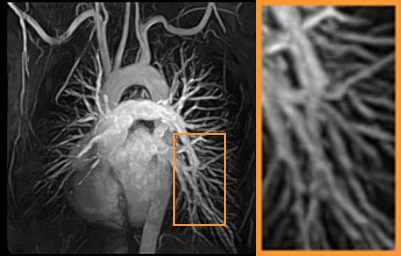}
				&\includegraphics[width=.18\textwidth]{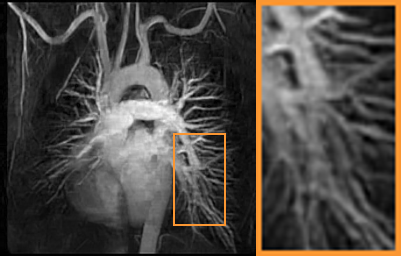}
				&\includegraphics[width=.18\textwidth]{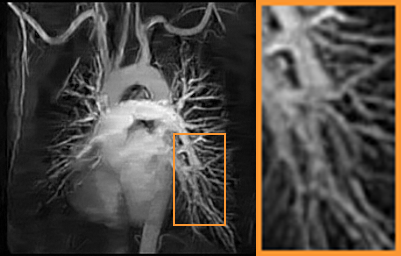}
				&\includegraphics[width=.18\textwidth]{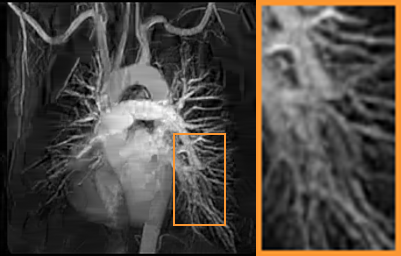}
				&\includegraphics[width=.18\textwidth]{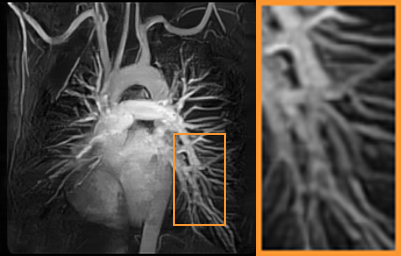}\\
				&Ground Truth (PSNR) 	&PANO (27.73) & FDLCP (26.84)& BM3D-MRI (26.27) & Ours (\textbf{28.22})\\			
			\end{tabular}
		\end{center}
		\caption{Qualitative comparisons on chest data using Cartesian mask with 30\% sampling rate.}
		\label{detail_chest}
	\end{figure*}
	
	\noindent{\textbf{Ablation Analysis:}}
	First we compare four different combinations of modules in our framework. The first one is to reconstruct directly with the prior module to figure out the role of a manual prior, the second one is to integrate the data-driven module with the fidelity module to explore the effect of data based distribution,  and  the third choice is combination of these three modules. Adding the optimal condition module, we get the entire paradigm as the last choice. For convenience, we refer them as $\mathcal{P}$, $\mathcal{F}\!\!\to\!\!\mathcal{N}$, $\mathcal{F}\!\!\to\!\!\mathcal{N}\!\!\to\!\!\mathcal{P}$ and \emph{Ours} respectively. We apply these strategies on $T_1$-weighted data using Ridial sampling pattern with 20\% sampling rate. The stopping criteria is set as $\|\mathbf{x}^{k+1} - \mathbf{x}^{k}\|/\|\mathbf{x}^{k}\| \leq 1e-4$. 
	
	As shows in Fig.~\ref{iteration}, at the first several iterations, the loss of $\mathcal{P}$ is slightly larger than that of $\mathcal{F}\!\!\to\!\!\mathcal{N}$. Because the input is corrupted with severe artifacts, thus the role of data-driven module is significant at the first several steps. But as process goes on, repeated denoising operation in turn causes over-smoothing. While module $\mathcal{P}$ can make up for it by incorporating model based knowledge. Though $\mathcal{F}\!\!\to\!\!\mathcal{N}\!\!\to\!\!\mathcal{P} $ can improve the performance, it cannot ideally converge to a desired solution. The solid line indicates the superiority of \emph{Ours} over other choices in both convergence rate and reconstruction accuracy.  The execution time of $\mathcal{P}$ , $\mathcal{F}\!\!\to\!\!\mathcal{N}$, $\mathcal{F}\!\!\to\!\!\mathcal{N}\!\!\to\!\!\mathcal{P} $ and \emph{Ours} is 4.4762s, 3.3240s, 6.2760s and 2.5225s, respectively. As expect, the proposed method provides a much faster reconstruction process. Thus we can verify that our framework has higher efficiency both in terms of theoretical convergence and practical execution time. 
	The visualized results in Fig.~\ref{component} also verify that Ours has better performance than others.
	
	\begin{figure}[t] 
		\begin{center}
			\begin{tabular}{c@{\extracolsep{-1em}}c@{\extracolsep{0.8em}}c@{\extracolsep{2em}}c}
				&\includegraphics[width=.22\textwidth]{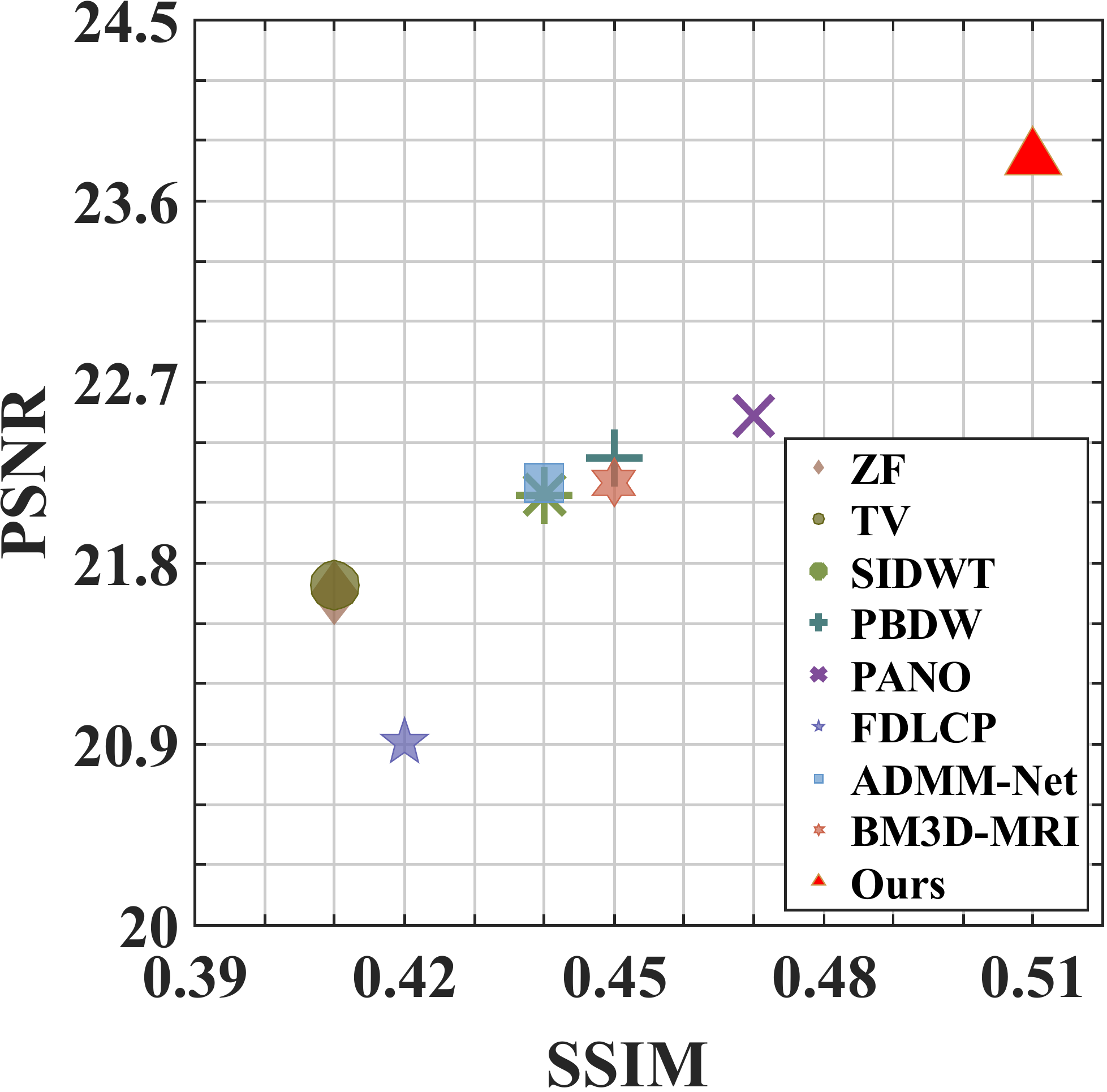}
				&\includegraphics[width=.225\textwidth]{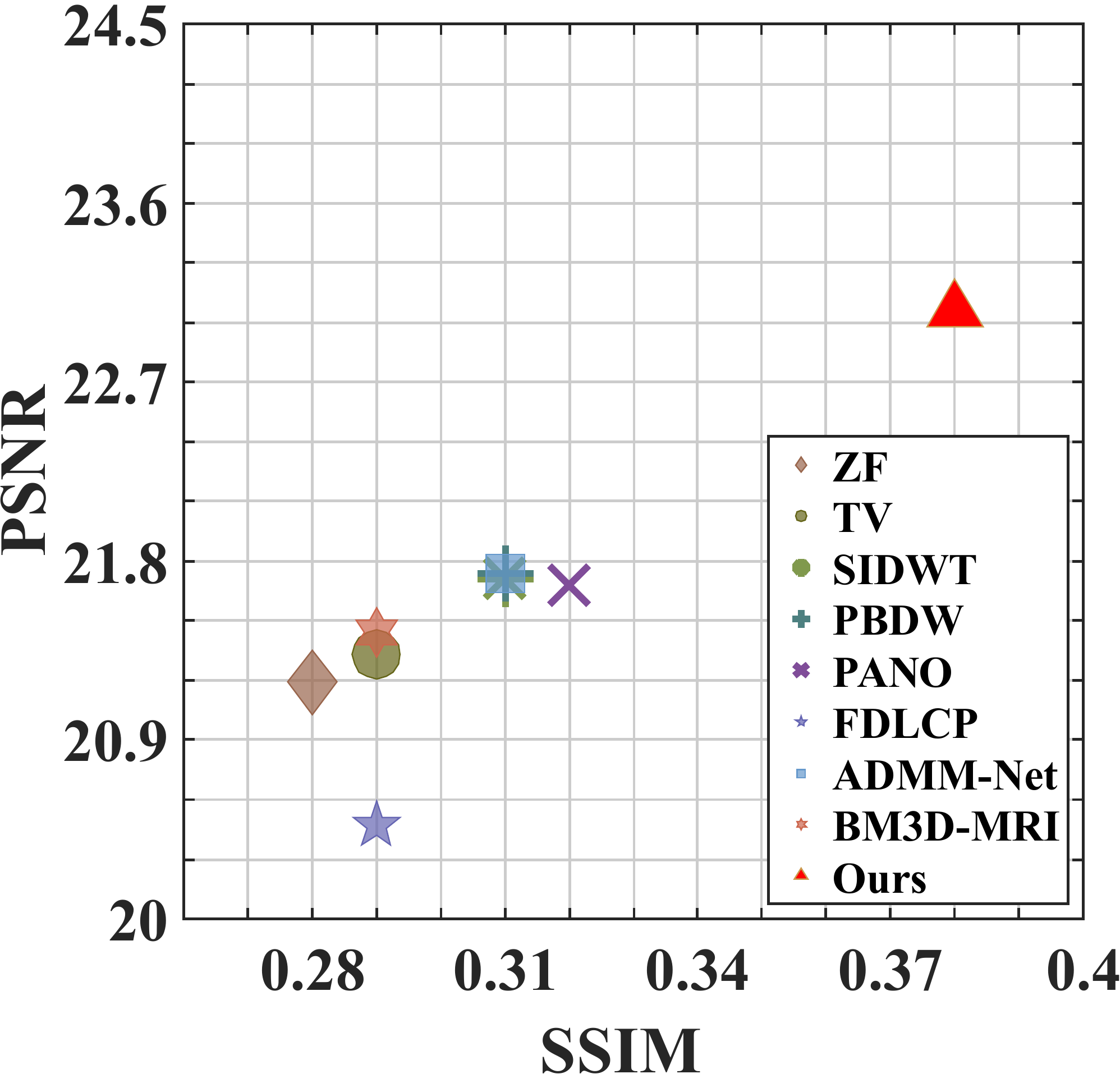}\\			
			\end{tabular}
		\end{center}
		\caption{Comparison of robustness. Left and right subfigures represent the results of $T_1$-weighted and $T_2$-weighted data, respectively.}
		\label{original}
	\end{figure}

	\begin{figure}[t]
		\begin{center}
			\begin{tabular}{c@{\extracolsep{0.3em}}c@{\extracolsep{0.3em}}c@{\extracolsep{0.3em}}c}
				\includegraphics[width=1.8cm,height=1.8cm]{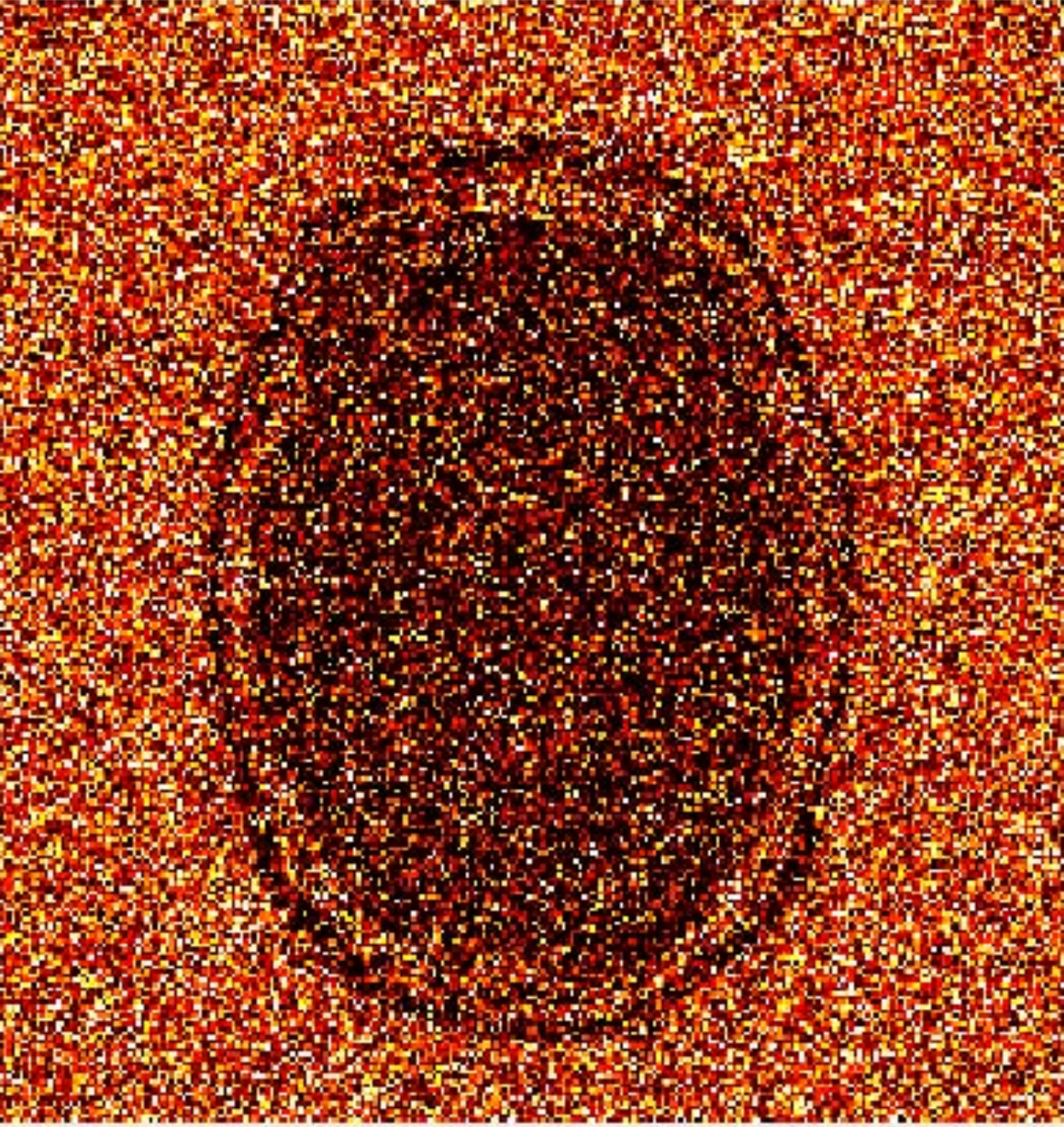}
				&\includegraphics[width=1.8cm,height=1.8cm]{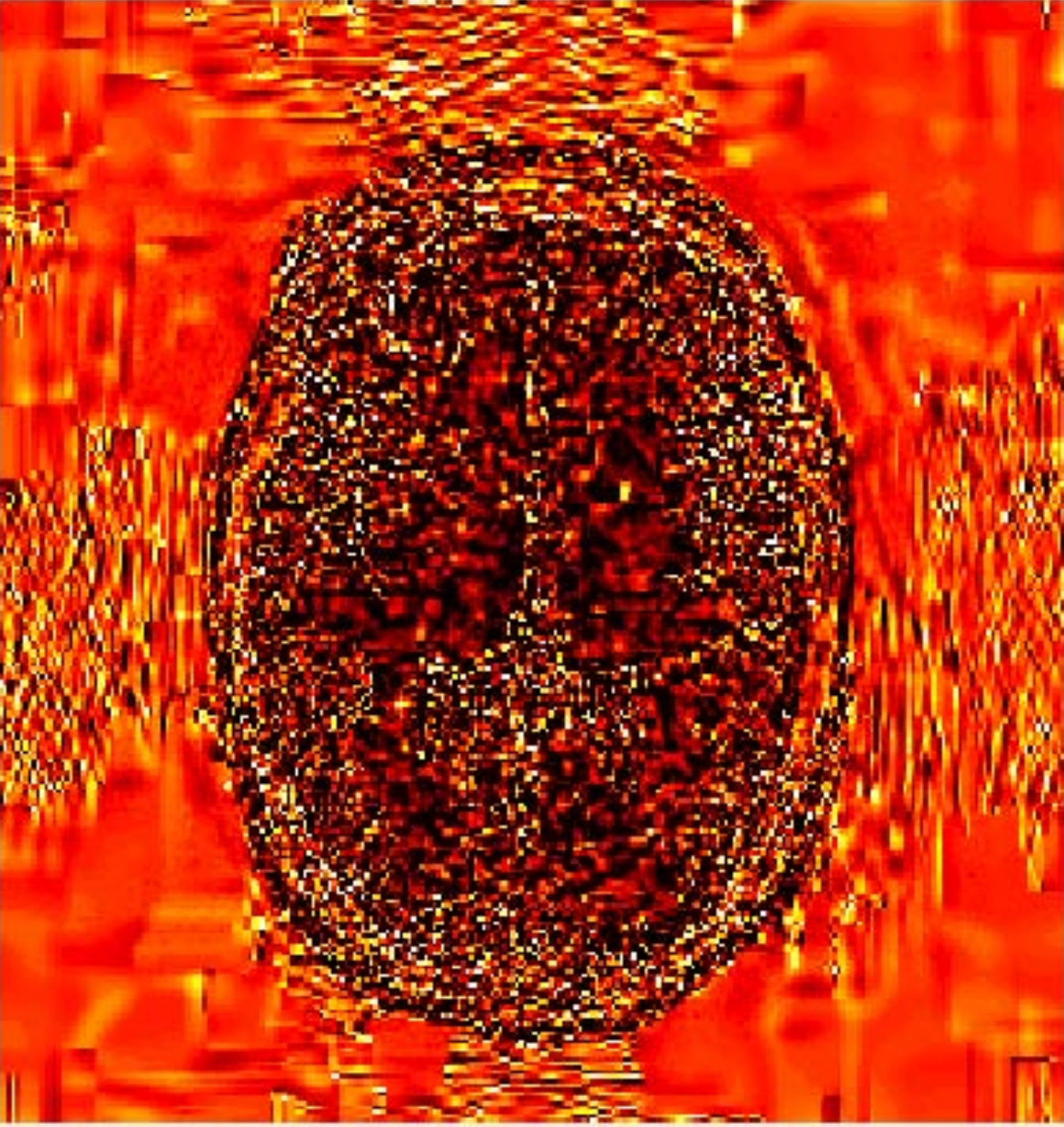}
				&\includegraphics[width=1.8cm,height=1.8cm]{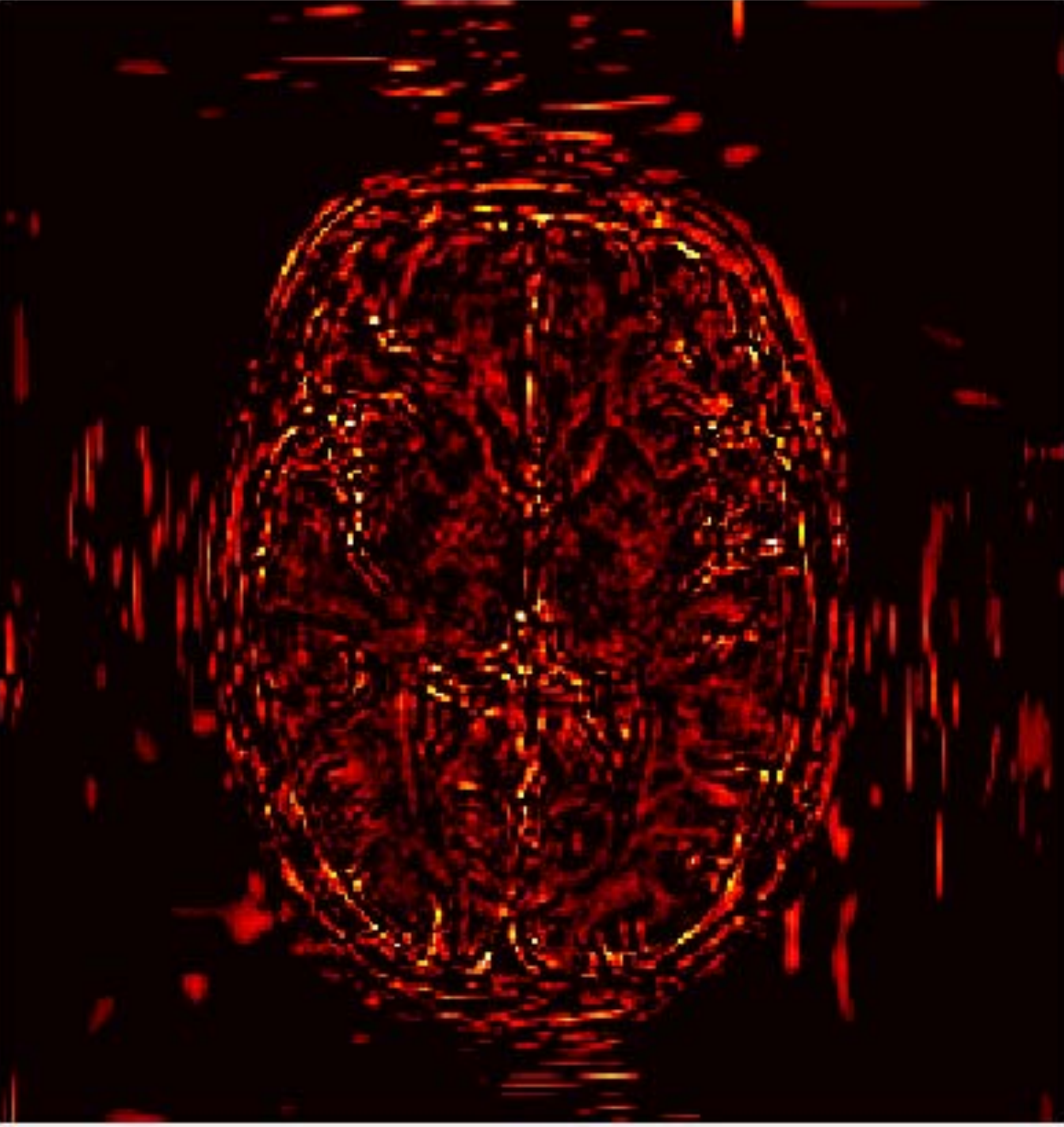}
				&\includegraphics[width=1.8cm,height=1.8cm]{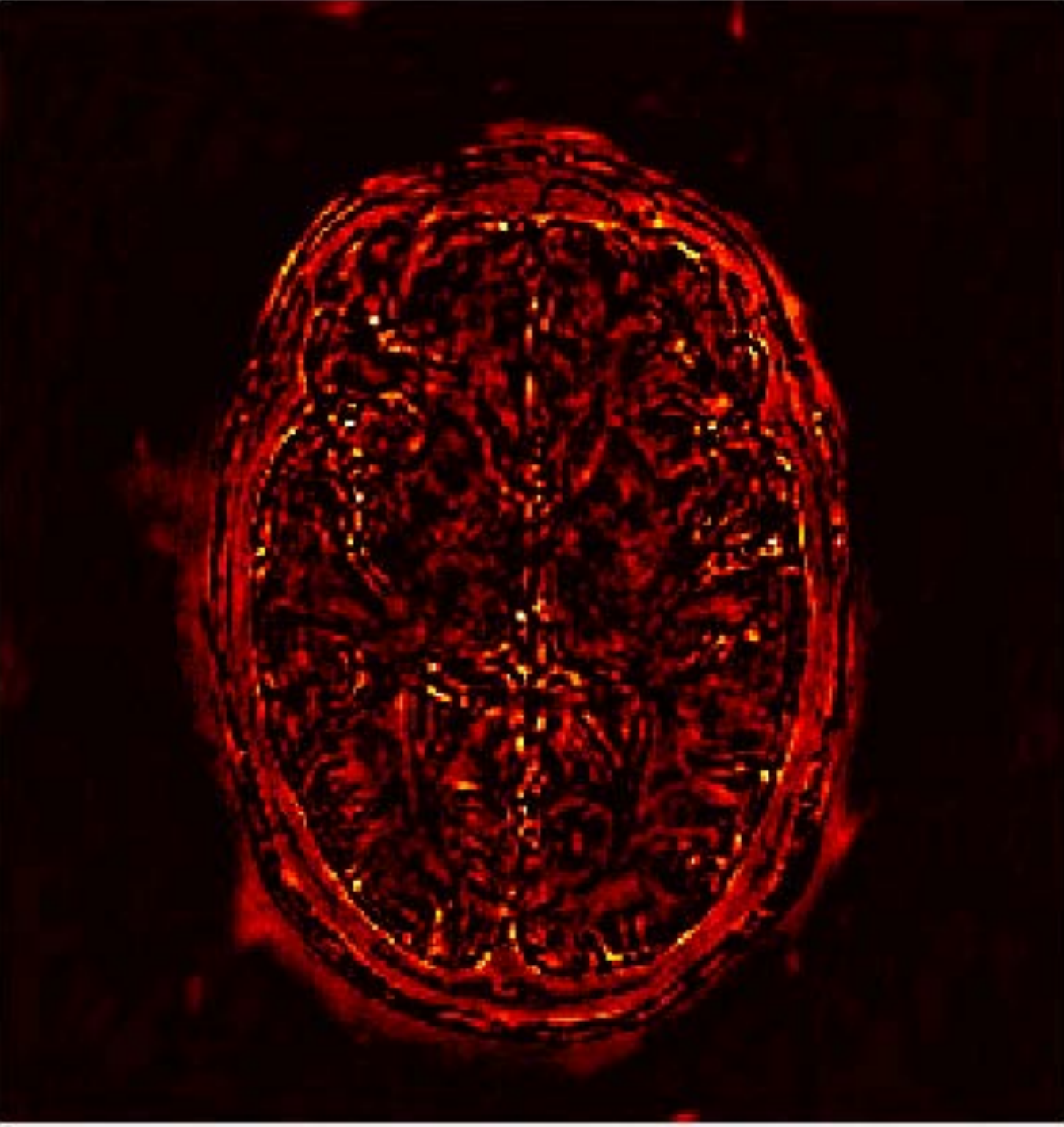}\\
				Noisy & BM3D-MRI & RiceOptVST & Our Denoiser
			\end{tabular}
		\end{center}
		\caption{Comparison between our denoiser (PSNR:31.72) and RiceOptVST (PSNR:30.34) for denoising.}
		\label{VST&RDN}
	\end{figure}
	
	\begin{table*}[t]
		\centering
		\caption{Comparison between our framework and traditional CS-MRI methods followed by a denoiser.}
		\begin{tabular}{c@{\extracolsep{0.5em}}c@{\extracolsep{0.5em}}c@{\extracolsep{0.5em}}c@{\extracolsep{0.5em}}c@{\extracolsep{0.5em}}c@{\extracolsep{0.5em}}c@{\extracolsep{0.5em}}c@{\extracolsep{0.5em}}c@{\extracolsep{0.5em}}c@{\extracolsep{0.5em}}c} 
			\toprule
			MRI Data & Denoiser & Zero-filling & TV & SIDWT & PBDW & PANO & FDLCP & ADMM-Net & BM3D-MRI & Ours \\
			\midrule
			\multirow{2}*{$T_1$-weighted}	& RiceOptVST &24.11  &	25.51  &	25.78  &	26.65 & 	27.08 	& 	25.66  &	26.36  &	27.36  &	\multirow{2}*{\textbf{28.05}} \\
			
			~							& Our Denoiser	&24.62  &	25.71  &    25.94  &	26.83 &		27.30 	&	25.59  &	26.49  &	27.64  &	~\\
			\midrule
			\multirow{2}*{$T_2$-weighted}	& RiceOptVST &25.40  &	27.83  &	28.37  &	29.13 &		29.22 	&	27.55  &	29.14  &	29.67  &	\multirow{2}*{\textbf{30.77}} \\
			~							& Our Denoiser		&26.36  &	28.20  &	28.68  &	29.46 & 	29.64  	&	27.44  &	29.45  &	30.34  &	~\\
			
			\bottomrule
		\end{tabular}
		\label{csnoise}
	\end{table*}
	\begin{figure*}[t]
		\begin{center}
			\begin{tabular}{c@{\extracolsep{0.2em}}c@{\extracolsep{0.2em}}
					c@{\extracolsep{0.2em}}c@{\extracolsep{0.2em}}c}
				\includegraphics[width=.18\textwidth]{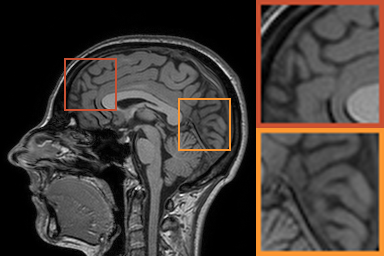}
				&\includegraphics[width=.18\textwidth]{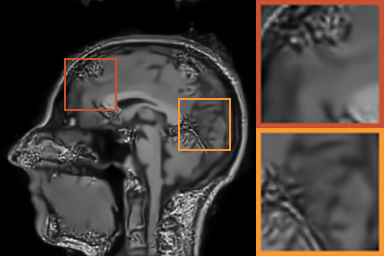}
				&\includegraphics[width=.18\textwidth]{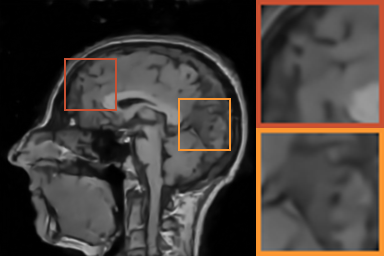}
				&\includegraphics[width=.18\textwidth]{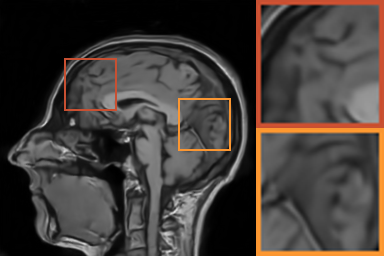}
				&\includegraphics[width=.18\textwidth]{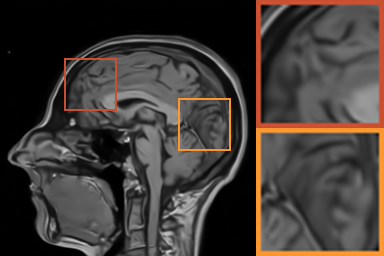}\\
				Ground Truth (PSNR)& FDLCP (25.72)& ADMM-Net (26.29)& BM3D-MRI (27.41)& Ours (\textbf{28.42})
			\end{tabular}
		\end{center}
		\caption{Qualitative comparison between our framework and traditional CS-MRI methods.}
		\label{detail2}
	\end{figure*}
	
	\noindent{\textbf{Comparisons with Existing Methods:}}
	In this section, we compare with three traditional CS-MRI approaches including Zero-filling \cite{bernstein2001effect}, TV \cite{lustig2007sparse} and SIDWT \cite{baraniuk2007compressive}, and five state-of-the-art like PBDW \cite{qu2012undersampled}, PANO \cite{qu2014magnetic}, FDLCP \cite{zhan2016fast}, ADMM-Net \cite{sun2016deep} and BM3D-MRI \cite{eksioglu2016decoupled}. 25 $T_1$-weighted MRI data and 25 $T_2$-weighted MRI data are randomly chosen from 50 subjects in IXI datasets\footnote{http://brain-development.org/ixi-dataset/} as the testing data.  In the experiment process, the parameter $\rho$ in module $\mathcal{F}$ is set as 5 and the noise level of network in module $\mathcal{N}$ ranges from 3.0 to 50.0. Parameter L ($\eta=2/L$), $\lambda$ and p in modules $\mathcal{C}$ and $\mathcal{P}$ are set as 1.1, 0.00001 and 0.8, respectively. The number of total iterations is 50. As for the parameters of comparative approaches, we adopt the most proper settings as suggested in their papers for fairness.
	
	First, we test on 25 T$_1$-weighted MRI data using three different undersampling patterns with a fixed 10\% sampling rate. Fig.~\ref{errorbar} shows the quantitative results (PSNR). Our method performances best for all three cases and has stronger stability compared with the second best method on variance. As for the effect of sampling ratios variation, we use radial mask under 10\%, 30\% and 50\% sampling rates with evaluation of RLNE and MSE. Fig.~\ref{error} shows that our method has the lowest reconstruction error for all sampling rates. For more intuitive comparison, we illustrate the reconstruction error in term of pixels in Fig.~\ref{heatmap}.  We also offer the qualitative comparison in Fig.~\ref{detail1}. Visualized results demonstrate our method has better performance in both artifacts removing and details restoration. Time consuming is also considered. We compare our method with others on the 25 T$_1$-weighted data using Radial mask with 10\% and 50\% sampling rate. Notice that ADMM-Net and ours are tested on GPU for the incorporation of deep architecture. Tab.~\ref{time} shows that our method  provides an efficient reconstruction process and comes to the fastest method among the state-of-the-art competitors.

	To demonstrate the robustness of our approach, we first apply it on various MRI data including the chest, cardiac and renal ~\cite{yazdanpanah2017compressed}. In Tab.~\ref{testingdata}, Our proposed framework gives the highest PSNR for all of the tree types of MR images. Fig.~\ref{detail_chest} visualizes the corresponding results for chest data. we can see that our approach prevails over others in detail restoration at the junction of blood vessels as well as noise removal in the background. Actually, our method has a stronger ability to handle slight noise because of the subprocess of learning based optimization with deep prior. To demonstrate that, we add Rician noise at level of 20 to 25 T$_1$-weighted MRI and 25 T$_2$-weighted MRI to generate the noisy data. As what is shown in Fig.~\ref{original}, our method over leads all the competitors by a large margin when the input is corrupted with Rician noise. 
	
	\begin{table}[!tp]
		\centering
		\caption{Comparison between four types of CNN denoiser.}
		\begin{tabular}{c@{\extracolsep{0.0em}}c@{\extracolsep{0.0em}}c@{\extracolsep{0.0em}}c@{\extracolsep{0.0em}}c@{\extracolsep{0.0em}}c} 
			\toprule
			& Denoiser$_1$ & Denoiser$_2$ & Denoiser$_3$ & Our Denoiser\\
			\midrule
			Denoising & \textbf{35.15} & 27.69 & 29.78 & 35.06\\
			Reconstruction   & 17.44 & 19.25 & 24.41 & \textbf{28.71}\\
			\bottomrule
		\end{tabular}
		\label{CNNdenoiser}
	\end{table}
	
	\subsection{Real-world CS-MRI}
	We further explore the performance of our approach on real-world CS-MRI with Rician noise and the parameters $\rho_1$, $\lambda_1$ and  $\lambda_2$ in Eq.~(13a) and Eq.~(13b) are set as 0.01, 1.0 and 1.0, respectively.

	\noindent{\textbf{Rician Network Behavior:}}
	In the learnable architecture, the first stage is to get $(\mathbf{x}_c+\mathbf{n}_1)^2$ and the second stage is to obtain the desired noise-free data $\mathbf{x}_c$ in Eq.~\eqref{eq:Rician}. At the training stage, we generate Rician noisy input data with $\sigma = 20$ using 500 $T_1$-weighted MR images randomly picked from the MICCAI 2013 grand challenge dataset\footnote{http://masiweb.vuse.vanderbilt.edu/workshop2013/index.php\\/Segmentation\_Challenge\_Details}.

	To verify the effectiveness of our learnable Rician network, we offer some other possible ways to obtain the $\mathbf{x}_c$ from $\mathbf{x}_n$ through deep learning. \emph{Denoiser$_1$} directly learns the difference between $\mathbf{x}_n$ and $\mathbf{x}_c$. \emph{Denoiser$_2$} learns the Gaussian noise existing in the real and imaginary parts separately. \emph{Denoiser$_3$} treats the Rician noise as a Gaussian one. For all the four types of denoisers, we use the same network architecture as IRCNN~\cite{zhang2017learning}. Tab.~\ref{CNNdenoiser} shows that \emph{Denoiser$_1$} gives comparable performance in denoising, but performs the worst for real-world CS-MRI. On the contrary, our learnable architecture gives a much better result than other methods. It is because the Rician noise is not additive noise. Directly estimation from the difference of $\mathbf{x}_c$ and $\mathbf{x}_n$ may cause error especially when the noise in the background is large. We also compare with the classical Rician denoising technique \emph{RicieOptVST}\footnote{http://www.cs.tut.fi/~foi/RiceOptVST/\#ref\_software} \cite{foi2011noise} to evaluate our learnable architecture. Fig.~\ref{VST&RDN} shows our learnable architecture has a better performance in removing the noise on background, indicating that we can also take the learnable architecture to address pure Rician denoising issues.

	\noindent{\textbf{Benchmark:}}
	We then compare our method with other CS-MRI techniques on the task of CS-MRI with noise. The $T_1$-weighted and $T_2$-weighted MRI data in IXI dataset are adopted as test benchmark. Since the compared methods don't have mechanism to handle Rician noise, we separately assign a classical Rician noise remover \emph{RicieOptVST} and our learnable architecture for them to execute the denoising after CS reconstruction.
	As shows in Tab.~\ref{csnoise}, our CS-MRI framework has superiority against others based on both \emph{RicieOptVST} and our network. Furthermore, the choice of taking our learnable architecture as the denoiser performs better than that of taking \emph{RiceOptVST} and the last column shows that the proposed framework surpasses all the combinations. 
	In Fig.~\ref{detail2}, we can have a more intuitive understanding to the reconstruction comparison. More details are preserved in our framework than competitive approaches.

	\section{Conclusions}
	We propose a theoretically converged deep optimization framework to efficiently solve the nonconvex and nonsmooth CS-MRI model. Our framework can 
	take advantage of fidelity, prior, data-driven architecture and optimal condition to guarantee the iterative variables converge to critical point of the specific model. For real-world CS-MRI with Rician noise, a learning based architecture is proposed for Rician noise removal. Experiments demonstrate that the our framework is robust and superior than others.
	
	\section{ Acknowledgments}
	This work is partially supported by the National Natural Science Foundation of China (Nos. 61672125, 61733002, 61572096 and 61632019), and the Fundamental Research Funds for the Central Universities.

\section{Supplemental Materials}\label{sec:proofs}
	To simplify the following derivations, we first rewrite the function in Eq.~\eqref{eq:OriginalModel} as
	$$
	\begin{array}{l}
	\Phi (\bm{\alpha}) = f(\bm{\alpha}) + g(\bm{\alpha}) =\frac{1}{2}\| \mathbf{PF\mathbf{A}\bm{\alpha}-y}\|_2^2+\lambda \|\bm{\alpha}\|_p.
	\end{array}
	$$
	We first provide detailed explanations about the properties of $f, g$, and $\Phi$.
	
	\begin{itemize}
		\item $f(\bm{\alpha})$ is proper if $\mathtt{dom}f:=\{\bm{\alpha}\in\mathbb{R}^D: f(\bm{\alpha})<+\infty\}$ is nonempty.
		\item $f(\bm{\alpha})$ is $L$-Lipschitz smooth if $f$ is differentiable and there exists $L>0$ such that 
		$$
		\|\nabla f(\bm{\alpha}) - \nabla f(\bm{\beta})\| \leq L \|\bm{\alpha} - \bm{\beta}\|, \ \forall \ \bm{\alpha},\bm{\beta} \in \mathbb{R}^{D}.
		$$
		If f is $L$-Lipschitz smooth, we have the following inequality
		$$
		f(\bm{\alpha})\!\leq\! f(\bm{\beta}) + \langle \nabla f(\bm{\beta}),\! \bm{\alpha}-\bm{\beta}\rangle + \frac{L}{2}\|\bm{\alpha}-\bm{\beta}\|^2, \ \forall \bm{\alpha}, \bm{\beta}\!\in\!\mathbb{R}^D.
		$$
		 \item $g(\bm{\alpha})$ is lower semi-continuous if $\liminf \limits_{\bm{\alpha}\to\bm{\beta}}g(\bm{\alpha})\geq g(\bm{\beta})$ at any point $\bm{\beta}\in\mathtt{dom}g$.
		\item $\Phi(\bm{\alpha})$ is coercive if $\Phi$ is bounded from below and $\Phi\to\infty$ if $\|\bm{\alpha}\|\to\infty$, where $\|\cdot\|$ is the $\ell_2$ norm.
		\item $\Phi(\bm{\alpha})$ is is said to have the Kurdyka-{\L}ojasiewicz (K\L) property at $\bar{\bm{\alpha}}\in\mathtt{dom}\partial \Phi:=\{\bm{\alpha}\in\mathbb{R}^D: \partial g(\bm{\alpha}) \neq \emptyset\}$ if there exist $\eta\in(0,\infty]$, a neighborhood $\mathcal{U}_{\bar{\bm{\alpha}}}$ of $\bar{\bm{\alpha}}$ and a desingularizing function $\phi:[0,\eta)\to \mathbb{R}_+$ which satisfies (1) $\phi$ is continuous at $0$ and $\phi(0)=0$; (2) $\phi$ is concave and $C^1$ on $(0,\eta)$; (3) for all $s\in(0,\eta): \phi'(s)>0$, such that for all
		$$
		\bm{\alpha}\in\mathcal{U}_{\bar{\bm{\alpha}}}\cap[\Phi(\bar{\bm{\alpha}})<\Phi(\bm{\alpha})<\Phi(\bar{\bm{\alpha}})+\eta],
		$$
		the following inequality holds
		$$
		\phi'(\Phi(\bm{\alpha})-\Phi(\bar{\bm{\alpha}}))\mathtt{dist}(0,\partial \Phi(\bm{\alpha})) \geq 1.
		$$
		Moreover, if $\Phi$ satisfies the K{\L} property at each point of $\mathtt{dom}\partial \Phi$ then $\Phi$ is called a K{\L} function.
	\end{itemize}
	\begin{proposition}\label{prop:c-error}
		Let $ \left\{\bm{\alpha}^k\right\}_{k\in\mathbb{N}} $ and  $\left\{\bm{\beta}^k\right\}_{k\in\mathbb{N}} $ be the sequences generated by Alg.\ref{alg1}. Suppose that the error condition 
		$\|\mathbf{v}^{k+1}-\bm{\alpha}^{k}\| \leq \varepsilon^{k}\|\bm{\beta}^{k+1}-\bm{\alpha}^{k}\|$ 
		in our $\mathtt{icheck}$ is satisfied. Then there existed a sequence $\{C^{k}\}_{k\in\mathbb{N}}$ such that 
		\begin{equation}	
		\Phi(\bm{\beta}^{k+1}) \leq \Phi(\bm{\alpha}^k)-C^{k}\|\bm{\beta}^{k+1}-\bm{\alpha}^k\|^2,
		\end{equation}
		where $C^{k} = \frac{1}{2\eta_{1}} - \frac{L_f}{2} -  (L_f  + |\rho-\frac{1}{\eta_{1}}|)\epsilon^{k} >0$ and $L_f$ is the Lipschitz coefficient of $\nabla f$ .
		\label{eq:ineq_fun_pgmomentum}
	\end{proposition}
	
	\begin{proof}
		In our $\mathtt{icheck}$ stage, we have 
		$$
		\bm{\beta}^{k+1}\!\in\!\emph{prox}_{\eta_1 \lambda\|\cdot\|_p}
		\!\left(\mathbf{v}^{k+1}\!-\!\eta_1\nabla f\left(\mathbf{v}^{k+1}\right)\!+	
		\!\rho\left(\mathbf{v}^{k+1}\!-\!\bm{\alpha}^{k}\right)\right),
		$$
		if the error condition is satisfied.
		Thus, by the definition of the proximal operator we get 
		\begin{equation}
		\begin{aligned}
		\bm{\beta}^{k+1} 
		&\!\in\!\arg\min\limits_{\bm{\beta}} \eta_{1}\lambda\|\bm{\beta} \|_{p}\! +\!\frac{1}{2}\|\bm{\beta}\! - \!\mathbf{v}^{k+1}\!+\!\eta_1 \left( \nabla f\left(\mathbf{v}^{k+1}\right)\right.\\
		&\left.\quad +\rho\left(\mathbf{v}^{k+1}\!-\bm{\alpha}^{k}\right) \right) \|^2 \\
		&= \arg\min\limits_{\bm{\beta}} \eta_{1}\lambda\|\bm{\beta} \|_{p} \!+
		\!\frac{1}{2}\|\bm{\beta} \!-\!\bm{\alpha}^{k}\!+\! \eta_1\nabla f\left(\mathbf{v}^{k+1}\right)	\\
		&\quad \  +(\eta_1\rho-1)\left(\mathbf{v}^{k+1}-\bm{\alpha}^{k}\right) \|^2\\
		&=\arg\min\limits_{\bm{\beta}} \lambda\|\bm{\beta} \|_{p} + \frac{1}{2\eta_{1}}\|\bm{\beta} -\bm{\alpha}^{k} \|^2 +	\\ &\quad \left\langle\bm{\beta} -	\bm{\alpha}^{k}, \nabla f\left(\mathbf{v}^{k+1}\right)+
		(\rho-\frac{1}{\eta_{1}})\left(\mathbf{v}^{k+1}-\bm{\alpha}^{k}\right)  \right\rangle.
		\end{aligned}
		\label{eq:min_func_pgmoment}
		\end{equation}
		Hence in particular, taking $\bm{\beta} = \bm{\beta}^{k+1}$ and $\bm{\beta} = \bm{\alpha}^{k}$ respectively, we obtain
		\begin{equation}
		\begin{array}{l}
		\lambda\|\bm{\beta}^{k+1}\|_{p} + \frac{1}{2\eta_{1}}\|\bm{\beta}^{k+1} -\bm{\alpha}^{k} \|^2 + \left\langle\bm{\beta}^{k+1} -\bm{\alpha}^{k}, \right.	\\
		\nabla f\left(\mathbf{v}^{k+1}\right)
		\left.+(\rho-\frac{1}{\eta_{1}})\left(\mathbf{v}^{k+1}-\bm{\alpha}^{k}\right)  \right\rangle 
		\leq \lambda\|\bm{\alpha}^{k}\|_{p}.
		\end{array}\label{eq:ineq_g}
		\end{equation}
		Invoking the Lipschitz smooth property of $f$, we also have
		\begin{equation}
		f(\bm{\beta}^{k+1})\! \leq\! f(\bm{\alpha}^{k})\!  +  \!\left\langle\bm{\beta}^{k+1}\!-\!\bm{\alpha}^{k},\! \nabla f(\bm{\alpha}^{k}) \right\rangle\! +\! \frac{L_f}{2}\| \bm{\beta}^{k+1} \!-\!\bm{\alpha}^{k}\|^2.\label{eq:ineq_f}
		\end{equation} 
		Combing Eqs.~\eqref{eq:ineq_g} and \eqref{eq:ineq_f}, we obtain
		$$
		\begin{array}{l}
		\begin{aligned}
		\!\!\!\Phi(\bm{\beta}^{k+1}) &\leq \Phi(\bm{\alpha}^{k})\! - \!\frac{1}{2\eta_{1}}\|\bm{\beta}^{k+1}\! -\!\bm{\alpha}^{k} \|^2 \!+\! \frac{L_f}{2}\| \bm{\beta}^{k+1} \!-\!\bm{\alpha}^{k}\|^2\\
		&\quad+ \left\langle \bm{\beta}^{k+1} -\bm{\alpha}^{k}, \nabla f(\bm{\alpha}^{k})  - \nabla f\left(\mathbf{v}^{k+1}\right)	\right.\\
		&\quad\left.-(\rho-\frac{1}{\eta_{1}})\left(\mathbf{v}^{k+1}-\bm{\alpha}^{k}\right)\right\rangle \\
		&\leq \Phi(\bm{\alpha}^{k})\! -\!  \frac{1}{2\eta_{1}}\|\bm{\beta}^{k+1}\!  -\! \bm{\alpha}^{k} \|^2 \! +\!  \frac{L_f}{2}\| \bm{\beta}^{k+1} \! -\! \bm{\alpha}^{k}\|^2 \\
		&\quad+(L_f  + |\rho-\frac{1}{\eta_{1}}|)\epsilon^{k} \|\bm{\beta}^{k+1} -\bm{\alpha}^{k} \|^2 \\
		&\leq \Phi(\bm{\alpha}^{k})\!  -\! (\frac{1}{2\eta_{1}}\!  - \! \frac{L_f}{2}\!  - \!  (L_f \!  +\!  |\rho\! -\! \frac{1}{\eta_{1}}|)\epsilon^{k})  \|\bm{\beta}^{k+1} \\
		&\quad-\bm{\alpha}^{k} \|^2 \\
		&\leq \Phi(\bm{\alpha}^{k}) -C^{k}\|\bm{\beta}^{k+1} -\bm{\alpha}^{k} \|^2.
		\end{aligned}
		\end{array}
		$$
		The last inequality holds under the 
		assumption $C^{k} = \frac{1}{2\eta_{1}} - \frac{L_f}{2} -  (L_f  + |\rho-\frac{1}{\eta_{1}}|)\epsilon^{k}>0$ in the $k$-th iteration. So far, this prove that the inequality ~\eqref{eq:ineq_fun_pgmomentum} in Proposition~\ref{prop:c-error} holds.
	\end{proof}
	
	\begin{proposition}\label{prop:pg}
		If $\eta_2 < 1/L_f$, let $ \left\{\bm{\alpha}^k\right\}_{k\in\mathbb{N}} $ and $\left\{\mathbf{w}^k\right\}_{k\in\mathbb{N}}$ be the sequences generated by a proximal operator in Alg.\ref{alg1}. Then we have
		\begin{equation}
		\Phi(\bm{\alpha}^{k+1}) \leq \Phi(\mathbf{w}^{k+1})-({1}/({2\eta_2}) - {L_f}/{2})\|\bm{\alpha}^{k+1}-\mathbf{w}^{k+1}\|^2.\label{eq:ineq_fun_pg}
		\end{equation}	
	\end{proposition}
	
	\begin{proof}
		As the proximal stage shows in Alg.~\ref{alg1}, we have 
		\begin{equation}
		\begin{split}
		\bm{\alpha}^{k+1}  &\in \emph{prox}_{\eta_2 \lambda\|\cdot\|_p}\left(\mathbf{w}^{k+1}-\eta_2\nabla f\left(\mathbf{w}^{k+1}\right)\right)\\
		&= \arg\min\limits_{\bm{\alpha}} \lambda \| \bm{\alpha}\|_{p} + \frac{1}{2\eta_2}\| \bm{\alpha} -\mathbf{w}^{k+1} \|^2 +\\
		&\quad \langle\bm{\alpha} -\mathbf{w}^{k+1} ,\nabla f(\mathbf{w}^{k+1}) \rangle
		\label{eq:min_func_pg}.
		\end{split}		
		\end{equation}
		Similar with the deduction in Proposition~\ref{prop:c-error}, we obtain the following inequality:
		$$
		\begin{array}{l}
		\lambda\| \bm{\alpha}^{k+1}\|_{p} + \frac{1}{2\eta_{2}}\|\bm{\alpha}^{k+1} -\mathbf{w}^{k+1}  \|^{2} + 
		\left\langle \bm{\alpha}^{k+1} -\mathbf{w}^{k+1}  , \right.\\
		\left.\nabla f(\mathbf{w}^{k}) \right\rangle \leq \lambda\| \mathbf{w}^{k+1}\|_{p}, \\\\
		f(\bm{\alpha}^{k+1}) \leq f(\mathbf{w}^{k+1}) + \left\langle \bm{\alpha}^{k+1} -\mathbf{w}^{k+1}, \nabla f(\mathbf{w}^{k+1}) \right\rangle  \\
		+\frac{L_f}{2}\| \bm{\alpha}^{k+1} -\mathbf{w}^{k+1}\|^2.
		\end{array}
		$$
		Thus we get the conclusion
		$$
		\Phi(\bm{\alpha}^{k+1}) \leq \Phi(\mathbf{w}^{k+1})-(\frac{1}{2\eta_2} - \frac{L_f}{2})\|\bm{\alpha}^{k+1}-\mathbf{w}^{k+1}\|^2.
		$$
		
	\end{proof}

	\begin{theorem}
		Suppose that $ \left\{\bm{\alpha}^k\right\}_{k\in\mathbb{N}} $ be a sequence generated by Alg.~\ref{alg1}. The following assertions hold.
		\begin{itemize}
			\item The square summable of sequence $\left\{\bm{\alpha}^{k+1}-\mathbf{w}^{k+1} \right\}_{k\in\mathbb{N}}$ is bounded, i.e., 
			$$
			\sum_{k=1}^{\infty}\|\bm{\alpha}^{k+1}-\mathbf{w}^{k+1}\|^2 < \infty.
			$$
			\item The sequence $ \left\{\bm{\alpha}^k\right\}_{k\in\mathbb{N}} $ converges to a critical point $\bm{\alpha}^{*}$ of $\Phi$.
		\end{itemize}
	\end{theorem}
	
	\begin{proof}
		We first verify that square summable of $\left\{\bm{\alpha}^k -\bm{w}^{k} \right\}_{k\in\mathbb{N}}$ is bounded. From Propositions~\ref{prop:c-error} and \ref{prop:pg}, we deduce the 
		$$
		\Phi(\bm{\alpha}^{k+1}) \leq \Phi(\mathbf{w}^{k+1}) \leq \Phi(\bm{\alpha}^{k})  \leq \Phi(\mathbf{w}^{k}) \leq \Phi(\bm{\alpha}^{0}),
		$$
		is established. It follows that both sequences $\left\{ \Phi(\bm{\alpha}^{k}) \right\}_{k\in\mathbb{N}}$ and $\left\{\Phi(\mathbf{w}^{k})\right\}_{k\in\mathbb{N}}$ are non-increasing. Then, since both $f$ and $g$ are proper, we have $\Phi$ is bounded and that the objective sequences $\left\{ \Phi(\bm{\alpha}^{k}) \right\}_{k\in\mathbb{N}}$ and $\left\{\Phi(\mathbf{w}^{k})\right\}_{k\in\mathbb{N}}$ converge to the same value $\Phi^{*}$, i.e.,
		$$
		\lim\limits_{k \to \infty} \Phi (\bm{\alpha}^{k} ) = \lim\limits_{k \to \infty} \Phi (\mathbf{w}^{k} ) = \Phi^{*}.
		$$
		Moreover, using the assumption $\Phi$ is coercive, we have that both $\left\{\bm{\alpha}^k \right\}_{k\in\mathbb{N}}$ and $\left\{\mathbf{w}^{k} \right\}_{k\in\mathbb{N}}$ are bounded and thus have accumulation points.
		Considering Eq.~\eqref{eq:ineq_fun_pg} and the relationship of $\Phi(\bm{\alpha})$ and $\Phi(\mathbf{w})$, we get for any $k\geq0$,
		\begin{equation}
		\begin{split}
		(\frac{1}{2\eta_{2}}-\frac{L_f}{2})\|\bm{\alpha}^{k+1}-\mathbf{w}^{k+1}\|^2 &\leq \Phi(\mathbf{w}^{k+1}) - \Phi(\bm{\alpha}^{k+1}) \\
		&\leq \Phi(\bm{\alpha}^{k}) -\Phi(\bm{\alpha}^{k+1}). \label{eq:phi_k_k+1}
		\end{split}
		\end{equation}
		Summing over $k$, we further have
		\begin{equation}
		(\frac{1}{2\eta_{2}}-\frac{L_f}{2}) \sum\limits_{k=0}^{\infty} \|\bm{\alpha}^{k+1}-\mathbf{w}^{k+1}\|^{2} \leq \Phi(\bm{\alpha}^{0}) -\Phi^{*}<\infty.\label{eq:ineq_phi}
		\end{equation}
		So far, the first assertion holds.
		
		Next, we prove $\lim\limits_{k \to \infty} \bm{\alpha}^{k} = \bm{\alpha}^{*}$ is a critical point of $\Phi$. Eq.~\eqref{eq:ineq_phi} implies that $\|\bm{\alpha}^{k+1}-\mathbf{w}^{k+1}\| \to 0$ when $k \to \infty$, i.e., there exist subsequences $\{\mathbf{w}^{k} \}_{k_j\in\mathbb{N}}$ and $\{\bm{\alpha}^{k} \}_{k_j\in\mathbb{N}}$ such that they share the same accumulation point $\bm{\alpha}^{*}$ as $j\to \infty$.
		Then by Eq.~\eqref{eq:min_func_pg}, we have 
		$$
		\begin{aligned}		
		&\lambda\| \bm{\alpha}^{k+1}\|_{p} \! +\!  \frac{1}{2\eta_{2}}\|\bm{\alpha}^{k+1}\!  -\! \mathbf{w}^{k+1}  \|^{2} \! +\!  \left\langle \bm{\alpha}^{k+1} \! -\! \mathbf{w}^{k+1} \!  ,\!  \nabla f(\mathbf{w}^{k}) \right\rangle \\	
		&\leq 
		\lambda\| \bm{\alpha}^{*}\|_{p} + \frac{1}{2\eta_{2}}\|\bm{\alpha}^{*} -\mathbf{w}^{k+1}  \|^{2} + \left\langle \bm{\alpha}^{*} -\mathbf{w}^{k+1}  , \nabla f(\mathbf{w}^{k}) \right\rangle.
		\end{aligned}
		$$
		Let $k+1 = k_j$ and $j \to \infty$,  then by taking $\lim\sup$ on the above inequality, we have 
		$$
		\lim\sup\limits_{j\to \infty} \| \bm{\alpha}^{k_j}\|_{p} \leq \| \bm{\alpha}^{*}\|_{p}.
		$$
		What's more, we also get $ \lim\inf\limits_{j\to \infty} \| \bm{\alpha}^{k_j}\|_{p} \geq \| \bm{\alpha}^{*}\|_{p}$ since $\|\cdot\|_{p}$ is lower semi-continuous. Thus we have $\lim\limits_{j\to \infty} \| \bm{\alpha}^{k_j}\|_{p} = \| \bm{\alpha}^{*}\|_{p}$.
		Considering the continuity of $f$, we have $\lim\limits_{j\to \infty} f(\bm{\alpha}^{k_j}) = f(\bm{\alpha}^{*}). $ Thus, we obtain
		\begin{equation}
		\lim\limits_{j\to \infty}\Phi(\bm{\alpha}^{k_j}) = \lim\limits_{j\to \infty} f(\bm{\alpha}^{k_j}) + \lambda \|\bm{\alpha}^{k_j}\|_{p} = \Phi(\bm{\alpha}^{*}).
		\label{eq:Phi_lim}
		\end{equation}
		By the first-order optimality condition of Eq.~\eqref{eq:min_func_pg} and $k_j = k+1$, we deduce
		$$
		0\in \partial \lambda\|\bm{\alpha}^{k_j}\|_p + \nabla f(\mathbf{w}^{k_j}) + \frac{1}{\eta_{2}} ( \bm{\alpha}^{k_j}-\mathbf{w}^{k_j}).
		$$
		Hence we get
		\begin{equation}
		\begin{array}{l}
		\quad \nabla f(\bm{\alpha}^{k_j})-\nabla f(\mathbf{w}^{k_j}) - \frac{1}{\eta_{2}} ( \bm{\alpha}^{k_j}-\mathbf{w}^{k_j}) \in \partial \Phi(\bm{\alpha}^{k_j}) 
		\Rightarrow \\\| \partial \Phi(\bm{\alpha}^{k_j})\| = \|\nabla f(\bm{\alpha}^{k_j})-\nabla f(\mathbf{w}^{k_j}) - \frac{1}{\eta_{2}} ( \bm{\alpha}^{k_j}-\mathbf{w}^{k_j}) \| \\
		\qquad \qquad  \quad  \leq (L_f + \frac{1}{\eta_{2}}) \|\bm{\alpha}^{k_j}-\mathbf{w}^{k_j}\|.
		\end{array}\label{eq:sub-diff_phi}
		\end{equation}
		Using the sub-differential of $\Phi$ and Eqs.~\eqref{eq:Phi_lim}, ~\eqref{eq:sub-diff_phi}, we finally deduce that $0 \in \Phi(\bm{\alpha}^{*})$, which means that $\{\bm{\alpha}^{k}\}_{k \in \mathbb{N}}$ is subsequence convergence.
		
		Furthermore, we will prove that $\{\bm{\alpha}^{k}\}_{k \in \mathbb{N}}$ is sequence convergence. Since $\Phi$ is a K{\L} function, we have
		$$
		\varphi'(\Phi(\bm{\alpha}^{k+1}) -\Phi(\bm{\alpha}^*) ) \mathtt{dist}(0, \partial \Phi(\bm{\alpha}^{k+1})) \geq 1.
		$$
		From Eq.~\eqref{eq:sub-diff_phi} we get that
		$$
		\varphi'(\Phi(\bm{\alpha}^{k+1}) -\Phi(\bm{\alpha}^*) ) \geq \frac{1}{L_f + \frac{1}{\eta_2}}\| \bm{\alpha}^{k_j}-\mathbf{w}^{k_j} \|^{-1}.
		$$
		On the other hand, from the concavity of $\varphi$ and Eq.~\eqref{eq:phi_k_k+1} and ~\eqref{eq:sub-diff_phi}
		we have that 
		$$
		\begin{array}{l}
		\varphi(\Phi(\bm{\alpha}^{k+1}) -\Phi(\bm{\alpha}^*) ) -
		\varphi(\Phi(\bm{\alpha}^{k+2}) -\Phi(\bm{\alpha}^*) ) \\
		\geq \varphi'(\Phi(\bm{\alpha}^{k+1}) -\Phi(\bm{\alpha}^*) ) (\Phi(\bm{\alpha}^{k+1}) - \Phi(\bm{\alpha}^{k+2}) ) \\
		\geq \frac{1}{L_f + \frac{1}{\eta_2}}\| \bm{\alpha}^{k+1}\! -\! \mathbf{w}^{k+1} \|^{-1}
		(\frac{1}{2\eta_2} \! - \! \frac{L_f}{2})
		\|\bm{\alpha}^{k+2} \! -\!  \mathbf{w}^{k+2}\|^{2}.
		\end{array}
		$$
		For convenience, we define for all $m,n \in \mathbb{N}$ and $\bm{\alpha}^{*}$ the following quantities
		$$
		\Delta_{m,n} := \varphi(\Phi(\bm{\alpha}^{m}) -\Phi(\bm{\alpha}^*) ) -
		\varphi(\Phi(\bm{\alpha}^{n}) -\Phi(\bm{\alpha}^*) ) ,
		$$
		and 
		$$
		E:= \frac{2L_f \eta_2+2}{1-L_f \eta_2}.
		$$
		These deduce that 
		$$
		\begin{array}{l}
		\quad \Delta_{k+1,K+2} \geq \frac{\|\bm{\alpha}^{k+2} - \mathbf{w}^{k+2}\|^2}{E \|\bm{\alpha}^{k+1}-\mathbf{w}^{k+1}\| }\\
		\Rightarrow \|\bm{\alpha}^{k+2} - \mathbf{w}^{k+2}\|^2 \leq E \Delta_{k+1,k+2} \|\bm{\alpha}^{k+1}-\mathbf{w}^{k+1}\|\\
		\Rightarrow 2 \|\bm{\alpha}^{k+2} \! -\!  \mathbf{w}^{k+2} \| \leq E \Delta_{k+1,k+2}\!  +\!  \|\bm{\alpha}^{k+1}\! -\! \mathbf{w}^{k+1}\|.
		\end{array}
		$$	
		Summing up the above inequality for $i=l,\dots,k$ yields
		$$
		\begin{aligned}
		&\sum\limits_{i=l+1}^{k} 2 \|\bm{\alpha}^{i+2} -\mathbf{w}^{i+2} \|\\
		&\leq \!\sum\limits_{i=l+1}^{k} \!\| \bm{\alpha}^{i+1}-\mathbf{w}^{i+1} \|+E \sum\limits_{i=l+1}^{k} \Delta_{i+1,i+2} \\ 
		&\leq \!\sum\limits_{i=l+1}^{k} \| \bm{\alpha}^{i+2}-\mathbf{w}^{i+2} \| + \|\bm{\alpha}^{l+2}-\mathbf{w}^{l+2} \| + E \Delta_{l+2,k+2},
		\end{aligned}		
		$$
		where the last inequality holds under the fact $\Delta_{m,n}+\Delta_{n,r} = \Delta_{m,r}$ for all $m,n,r \in \mathbb{N}$.
		Since $\varphi>0$, we thus have for any $k>l$ that
		\begin{equation}
		\begin{aligned}
		&\sum\limits_{i=l+1}^{k}  \|\bm{\alpha}^{i+2} - \mathbf{w}^{i+2} \|
		\leq \|\bm{\alpha}^{l+2}-\mathbf{w}^{l+2} \| + E \Delta_{l+2,k+2} \\
		&\leq \|\bm{\alpha}^{l+2}-\mathbf{w}^{l+2} \| + E \varphi(\Phi(\bm{\alpha}^{l+2}) - \Phi(\bm{\alpha}^{*})).
		\label{eq:a2w2}
		\end{aligned}
		\end{equation}
		Moreover, recalling the conclusion in Proposition~\ref{prop:c-error},  we also has
		\begin{equation}
		\begin{split}
		&\!\min\limits_{i}\{C^{i}\}\! \sum\limits_{i=l+1}^{k} \!\| \mathbf{w}^{i+2} \!-\! \bm{\alpha}^{i+1}\|\! \leq\! \sum\limits_{i=l+1}^{k}\! (\Phi(\mathbf{w}^{i+2})\! -\!\Phi(\bm{\alpha}^{i+1}\!) ) \\
		&\leq  \sum\limits_{i=l+1}^{k} (\Phi(\bm{\alpha}^{i+2}) -\Phi(\bm{\alpha}^{i+1}) ) = \Phi(\bm{\alpha}^{k+2}) - \Phi(\bm{\alpha}^{l+2}).
		\label{eq:w2a1}
		\end{split}
		\end{equation}
		Combing with Eqs.~\eqref{eq:a2w2} and \eqref{eq:w2a1}, we easily deduce
		\begin{equation}
		\sum\limits_{k=1}^{\infty} \|\bm{\alpha}^{k+1} - \bm{\alpha}^{k} \| < \infty.
		\label{eq:cauchy}
		\end{equation}
		It is clear that Eq.~\eqref{eq:cauchy} implies that the sequence $\{ \bm{\alpha}^{k}\}_{k \in \mathbb{N}}$ is a Cauchy sequence and hence is a convergent sequence. 
		So far, the second assertion holds.
		
	\end{proof}
	
	\section{Sampling Masks}
	In experiments we include three common used types of undersampling masks such as the Cartesian pattern in \cite{qu2012undersampled}, Radial pattern in \cite{sun2016deep}  and Gaussian mask in \cite{yang2018dagan}. Fig.\ref{mask} gives a visualization of the three kinds of patterns at a unified sampling ratio of 30\%.
	\begin{figure}[!htbp]
		\begin{center}
			\begin{tabular}{c@{\extracolsep{0.8em}}c@{\extracolsep{0.8em}}c}
				\includegraphics[width=.14\textwidth]{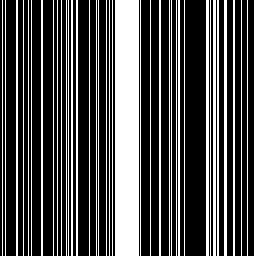}
				&\includegraphics[width=.14\textwidth]{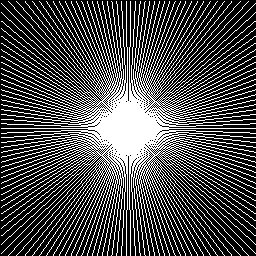}
				&\includegraphics[width=.14\textwidth]{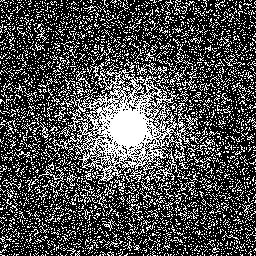}\\
				Cartesian & Radial & Gaussian
			\end{tabular}
		\end{center}
		\caption{Three types of sampling pattern.}
		\label{mask}
	\end{figure}

\bibliographystyle{aaai}
\bibliography{reference1}

\end{document}